\documentclass[11pt]{amsart}
\usepackage{geometry}                
\usepackage{url}
\usepackage{graphicx}
\usepackage[algoruled,boxed,lined,ruled,linesnumbered]{algorithm2e}
\SetKwInOut{Parameter}{Parameters}
\SetKwComment{Comment}{/* }{ */}
\usepackage{amssymb}
\usepackage{epstopdf}
\usepackage{color}
\usepackage{amsmath}
\usepackage{amsthm}
\usepackage{url}
\usepackage[square,numbers]{natbib}
\usepackage{subeqnarray}
\newcommand*{\loc}{\alpha}
\newcommand*{\locind}{{[\alpha]}}

\theoremstyle{definition}
\newtheorem{rem}{Remark}
\newtheorem{lemma}{Lemma}

\title{Localized Schr\"odinger Bridge Sampler}
\author{Georg A.~Gottwald}
\address{School of Mathematics and Statistics, University of Sydney}
\email{georg.gottwald@sydney.edu.au}
\author{Sebastian Reich}
\address{Department of Mathematics, University of Potsdam}
\email{sebastian.reich@uni-potsdam.de}

\date{\today}

\begin{document}

\begin{abstract} We consider the problem of sampling from an unknown distribution for which only a sufficiently large number of training samples are available. In this paper, we build on previous work combining Schr\"odinger bridges and plug \& play Langevin samplers. A key bottleneck of these approaches is the exponential dependence of the required training samples on the dimension, $d$, of the ambient state space. We propose a localization strategy which exploits conditional independence of conditional expectation values. Localization thus replaces a single high-dimensional Schr\"odinger bridge problem by $d$ low-dimensional Schr\"odinger bridge problems over the available training samples. In this context, a connection to multi-head self attention transformer architectures is established. As for the original Schr\"odinger bridge sampling approach, the localized sampler is stable and geometric ergodic. The sampler also naturally extends to conditional sampling and to Bayesian inference. We demonstrate the performance of our proposed scheme through experiments on a high-dimensional Gaussian problem, on a temporal stochastic process, and on a stochastic subgrid-scale parametrization conditional sampling problem. We also extend the idea of localization to plug \& play Langevin samplers using kernel-based denoising in combination with Tweedie's formula.
\end{abstract}

\maketitle

\small{

\smallskip

\noindent
{\bf AMS:} 60H10,62F15,62F30,65C05,65C40}

%
\section{Introduction}
%

In this paper, we consider the problem of sampling from an unknown probability measure $\nu({\rm d x})$
on $\mathbb{R}^d$ for which we only have access to a finite set of training samples $x^{(j)} \sim \nu$, $j=1,\ldots,M$. This problem has recently attracted widespread interest in the context of score-generative or diffusion modeling \citep{song2019generative,diffusion1,diffusion2,diffusion3,diffusion4}. If the probability measure $\nu({\rm d}x)$ possesses a probability density function $\pi(x)$, then a popular non-parametric approach to generative modeling is to estimate the score function $s (x;\theta)\approx \nabla \log \pi(x)$ by minimizing an appropriate loss function such as
\begin{equation} \label{eq:loss}
    \mathcal{L}(\theta) = \int_{\mathbb{R}^d}  \|s(x;\theta)-\nabla \log \pi(x)\|^2 \pi(x) {\rm d}x
\end{equation}
in the parameters $\theta \in \mathbb{R}^{d_\theta}$ \citep{hyvarinen2005estimation}. This estimate can then be used in combination with overdamped Langevin dynamics to yield 
\begin{equation} \label{eq:score_based}
\dot{X}(\tau) = s(X(\tau);\theta) + \sqrt{2}\,\dot{W}(\tau),
\end{equation}
where $W(\tau)$ denotes standard $d$-dimensional Brownian motion  \citep{Pavliotis2016}. The stochastic differential equation is typically discretized by the Euler--Maruyama (EM) method to yield an iterative update of the form
\begin{equation} \label{eq:EM_SDE}
X(n+1) = X(n) + \epsilon s(X(n);\theta) + \sqrt{2\epsilon}\, \Xi(n), \qquad \Xi(n) \sim {\rm N}(0,I),
\end{equation}
for $n\ge 0$, where $\epsilon >0$ denotes the step size and $X(n)$ provides the numerical approximation to the solution of (\ref{eq:score_based}) at time $\tau_n = n\,\epsilon$. The EM algorithm is initialized at one of
the training data points; i.e., $X(0) = x^{(j^\ast)}$ with $j^\ast \in \{1,\ldots,M\}$ 
appropriately chosen, and the resulting discrete trajectory $X(n)$, $n\ge 1$, delivers approximate samples from the target distribution $\pi(x)$.

Instead of first estimating the score function from samples and then discretizing~\eqref{eq:score_based} in time, it has been proposed in \cite{GLRY24} to employ Schr\"odinger bridges and to directly estimate the conditional expectation value 
\begin{equation} \label{eq:CE}
\mu (x;\epsilon) := \mathbb{E}[X(\epsilon)|X(0)=x]
\end{equation}
from the given samples $\{x^{(j)}\}_{j = 1}^M$ for given time-step $\epsilon >0$. We denote the Schr\"odinger bridge approximation obtained from the samples by $m(x;\epsilon):\mathbb{R}^d \times \mathbb{R}_+ \to \mathbb{R}^d$ and obtain the iteration scheme
\begin{equation} \label{eq:propagator_general}
X(n+1) = m(X(n);\epsilon) + \sqrt{S(X(n);\epsilon)} \,\Xi(n),
\qquad \Xi(n) \sim {\rm N}(0,I),
\end{equation}
with appropriately defined diffusion matrix $S(x;\epsilon) \in \mathbb{R}^{d\times d}$. Broadly speaking, $m(x;\epsilon)$ controls the drift while $S(x;\epsilon)$ moderates the noise. An obvious choice for $S(x;\epsilon)$ is $S(x;\epsilon) = 2\epsilon I$, which corresponds to the EM discretization (\ref{eq:EM_SDE}). A data-aware $S(x;\epsilon)$ has been introduced in \cite{GLRY24}, which is defined as the Schr\"odinger bridge approximation to the covariance matrix
\begin{equation} \label{eq:Sigma}
\Sigma(x;\epsilon) := \mathbb{E}[X(\epsilon)X(\epsilon)^\top|X(0)=x] - \mu (x;\epsilon)\mu (x;\epsilon)^\top.
\end{equation}
Provided the measure $\nu({\rm d}x)$ possesses a smooth density $\pi(x)$, it holds asymptotically that
\begin{equation}
\Sigma(x;\epsilon) = 2\epsilon I + \mathcal{O}(\epsilon^2).
\end{equation}

\begin{rem}
We emphasize that the discrete-time formulation (\ref{eq:propagator_general}) can be considered even in case the probability measure $\nu({\rm d}x)$ does not possess a probability density function $\pi(x)$ with respect to the Lebesgue measure on $\mathbb{R}^d$; e.g., the measure $\nu$ is concentrated on a submanifold $\mathcal{M}\subset \mathbb{R}^d$, as long as the conditional expectation values (\ref{eq:CE}) and (\ref{eq:Sigma}) can be defined appropriately. The Schr\"odinger bridge approximation allows for such an extension \citep{GLRY24}. Indeed, the, so called, {\it manifold hypothesis} states that many applications of generative modeling lead to measures $\nu({\rm d}x)$ which concentrate on a low-dimensional manifold $\mathcal{M}$ in a high-dimensional ambient space $\mathbb{R}^d$ \cite{FeffermanEtAl16,diffusion2,WhiteleyEtAl24}.
\end{rem}

\noindent
We note that (\ref{eq:propagator_general}) is closely related to the plug \& play unadjusted Langevin sampler (PnP-ULA) of \cite{plugandplay}, where a denoiser $D(x;\epsilon)$ takes the role of $m(x;\epsilon)$ in (\ref{eq:propagator_general}), which should satisfy
\begin{equation} \label{eq: denoiser based score}
\nabla \log \pi(x) \approx \frac{D(x;\epsilon)-x}{\epsilon}
\end{equation}
and $S(x;\epsilon) = 2\epsilon I$. An additional stabilizing term of the form
\begin{equation}
\frac{\epsilon}{\lambda}(P_\mathcal{C}(X(n))-X(n))
\end{equation}
is required  for the associated PnP-ULA scheme to satisfy an appropriate growth condition. Here $\lambda >0$ is a suitable parameter and $P_\mathcal{C}(x)$ projects $x$ onto a compact set $\mathcal{C} \subset \mathbb{R}^d$ which should contain most of the probability mass of $\nu({\rm d}x)$. In particular, if $\nu({\rm d}x)$ is supported on a manifold $\mathcal{M} \subset \mathbb{R}^d$, then $\mathcal{C} \subseteq \mathcal{M}$. See \cite{plugandplay} for more details and \cite{DMP18} for a related approach using a Moreau--Yoshida regularised score function. We note that the Schr\"odinger bridge sampler (\ref{eq:propagator_general}) has been shown to be stable and geometric ergodic \cite{GLRY24} without any additional stabilization term. 

The sampler (\ref{eq:propagator_general}) can be used in the general context of score-generative or diffusion modeling, however, our main motivation is in Bayesian inference and in conditional sampling with applications to multi-scale processes. Applications to Bayesian inference, for which $\nu({\rm d}x)$ takes the role of the prior for given likelihood function 
$\pi(y|x)$, immediately suggest the modified update
\begin{equation} \label{eq:discrete Bayes}
X(n+1) = m(X(n);\epsilon) +\epsilon \nabla \log \pi(y|X(n)) + \sqrt{S(X(n);\epsilon)} \,\Xi(n),
\qquad \Xi(n) \sim {\rm N}(0,I).
\end{equation}
Furthermore, a particular choice of $\pi(y|x)$ can be used for conditional sampling \cite{GLRY24}. We note that (\ref{eq:discrete Bayes}) fits into the general plug \& play approach to data-aware Bayesian inference \cite{PnP,AMOS19,plugandplay}.

While it has been demonstrated in \cite{GLRY24} that (\ref{eq:propagator_general}) and (\ref{eq:discrete Bayes}) work well for low-dimensional problems, the required number of training samples, $M$, increases exponentially in the dimension, $d$, of the samples \cite{WR20}. In order to remedy this manifestation of the curse of dimensionality, we propose to utilize conditional independence in order to replace the Schr\"odinger bridge estimator for the conditional expectation value $m(x;\epsilon) \in \mathbb{R}^d$ by appropriately localized Schr\"odinger bridge estimators in each of the $d$ components of $m(x;\epsilon)$ and similarly for the diffusion matrix $S(x;\epsilon)$. The proposed localization strategy resembles localization strategies used in the ensemble Kalman filter (EnKF) \citep{Evensenetal2022,reich2015probabilistic,asch2016data}; but is fundamentally different in at least two ways: (i) For Gaussian measures with covariance matrix $C$, the EnKF would localize the empirical estimator of $C$ while our approach relies on the localization of the precision matrix $C^{-1}$ as dictated by conditional independence. (ii) Localized Schr\"odinger bridge estimators are not restricted to Gaussian measures as long as conditional independence can be established. Furthermore, we extend the proposed localization strategy to conditional mean estimators based on kernel denoising \cite{MD24}.

The paper is organized as follows. The Schr\"odinger bridge formulation for $m(x;\epsilon)$ and $S(x;\epsilon)$ in (\ref{eq:propagator_general}) is summarized in the subsequent Section \ref{sec:PaP LD}. There we also discuss connections to minimum mean square error (MMSE) denoising and kernel-based denoising, in particular to \cite{MD24}. The localized variant is subsequently developed in Section \ref{sec:lsbs} first for a Gaussian distribution for which $C^{-1}$ has a tri-diagonal structure and then for general target measure $\nu({\rm d}x)$ for which conditional independence holds. An algorithmic summary is provided in Algorithm \ref{alg:LSBS} and a discussion of numerical properties is provided in Section \ref{sec:properties}. Localization is extended to kernel-based denoising in Section \ref{sec:local KDE}. We discuss a connection between the Schr\"odinger bridge sampler and  self attention transformers \cite{Transformer} in Remark \ref{rem:2} and its localized variant in the context of multi-head transformers in Remark \ref{rem:3}. As applications, we consider sampling temporal stochastic processes in Section \ref{sec:temporal} and conditional sampling for a closure problem arising from the multi-scale Lorenz-96 model \citep{Lorenz96} in Section \ref{sec:cond}. The paper closes with some conclusions and suggestions for further work.

%
\section{Plug \& play Langevin sampler} \label{sec:PaP LD}
%

In this section, we summarize two particular variants of plug \& play Langevin samplers \cite{plugandplay}.
The first sampler has been proposed in \cite{GLRY24} and is based on a Schr\"odinger bridge approximation of the Langevin semi-group with invariant measure $\nu ({\rm d}x)$ \cite{WR20}. The second sampler builds upon kernel denoising \cite{MD24} and Tweedie's formula \cite{tweedie}.

%
\subsection{Schr\"odinger bridge sampler} \label{sec:SB sampler}
%

In this subsection, we briefly recall how to approximate the conditional estimates (\ref{eq:CE}) and (\ref{eq:Sigma}) using Schr\"odinger bridges. One first introduces the symmetric matrix $T \in \mathbb{R}^{M\times M}$ of (unnormalized) transition probabilities
\begin{equation} \label{eq:tij}
(T)_{jk} = \exp \left( -\frac{1}{4\epsilon} \|x^{(k)}-x^{(j)}\|^2 \right)
\end{equation}
for $j,k = 1,\ldots,M$. See \cite{GLRY24} for a more general definition involving a state-dependent scaling matrix $K(x)$ and variable bandwidth implementation $K(x) = \rho(x)I$ with $\rho(x)>0$ a suitable scaling function.

One next introduces the uniform probability  vector 
$w^\ast = (1/M,\ldots,1/M)^{\top} \in \mathbb{R}^M$ over the samples $\{x^{(j)}\}_{j=1}^M$. The associated Schr\"odinger bridge problem can be reformulated into finding the non-negative scaling vector $v \in \mathbb{R}^M$ such that the symmetric matrix
\begin{equation} \label{eq:DM_approximation}
P = D(v) \,T \,D(v)
\end{equation}
is a Markov chain with invariant distribution $w^\ast$, i.e.,
\begin{equation}
P \,w^\ast = w^\ast.
\label{eq:PDTD}
\end{equation}
Here $D(v) \in \mathbb{R}^{M\times M}$ denotes the diagonal matrix with diagonal entries provided by $v \in \mathbb{R}^M$. We remark that the standard scaling used in Schr\"odinger bridges would lead to a bistochastic matrix $\tilde P$, which is related to (\ref{eq:DM_approximation}) by $\tilde P = M^{-1} P$ \cite{PeyreCuturi}.

The next step is to extend the discrete Markov chain (\ref{eq:DM_approximation}) to all $x\in \mathbb{R}^{d}$. For that purpose one introduces the vector-valued function $t(x) \in \mathbb{R}^M$ with entries
\begin{equation}
\label{eq:tvec}
t^{(j)}(x) =  \exp \left( -\frac{1}{4\epsilon} \|x-x^{(j)}\|^2 \right)
\end{equation}
for $j = 1,\ldots,M$. One then defines the probability vector $w(x) \in \mathbb{R}^M$ using the Sinkhorn weights, $v$, obtained in~\eqref{eq:DM_approximation}, i.e., 
\begin{equation} \label{eq:probability vectors}
w (x) = \frac{D(v) \,t(x)}{v^{\top} t (x)} \in \mathbb{R}^M
\end{equation}
for all $x\in \mathbb{R}^d$. This vector gives the transition probabilities from any $x$ to the data samples, which we collect in the data matrix of samples
\begin{equation}\label{eq:data_vec}
\mathcal{X} = (x^{(1)},\ldots,x^{(M)}) \in \mathbb{R}^{d\times M}.
\end{equation}
Hence, the desired sample-based approximation of the conditional mean is given by
\begin{equation} \label{eq:mean}
m(x;\epsilon) := \mathcal{X} \,w (x),
\end{equation}
which provides a finite-dimensional approximation of the conditional expectation value $\mu(x;\epsilon)$
of the true underlying diffusion process (\ref{eq:score_based}). Note that the conditional mean $m(x;\epsilon)$ lies in the convex hull of the data since $0\le w(x)\le 1$ is a probability vector for all $x$

We also recall a data-aware choice of the covariance matrix $S (x;\epsilon)$  \cite{GLRY24}. Using~\eqref{eq:probability vectors} and~\eqref{eq:data_vec}, one can define the conditional covariance matrix
\begin{equation} \label{eq:cm_estimate}
S(x;\epsilon) = \mathcal{X}\,D(w(x))\,  \mathcal{X}^{\top} -
m(x;\epsilon)\, m(x;\epsilon)^{\top} \in
\mathbb{R}^{d\times d},
\end{equation}
which is the empirical covariance matrix associated with the probability vector $w(x)$; compare (\ref{eq:Sigma}).

It has been found advantageous in \cite{GLRY24} to replace the time-stepping method
(\ref{eq:propagator_general}) by the split-step scheme
\begin{subequations} \label{eq:update_ss}
    \begin{align}
        X(n+1/2) &= X(n) + \sqrt{S(X(n);\epsilon)}\, \Xi(n),\qquad \Xi(n) \sim {\rm N}(0,I),\\
        X(n+1) &=  m (X(n+1/2);\epsilon),
    \end{align}
\end{subequations}
which can be viewed as sequential noising and denoising steps. The key property of the Schr\"odinger bridge sampler is that the final step of the Langevin sampler (\ref{eq:update_ss}b) amounts to a projection into the convex hull of the samples, independent of the outcome of the noising step (\ref{eq:update_ss}a). This renders the sampling scheme numerically stable for any finite sample size $M$. This is in contrast to traditional Langevin samplers such as score generative models which directly solve the typically stiff Langevin equation \eqref{eq:score_based}; e.g., in case the probability measure $\nu({\rm d}x) $ concentrates on a submanifold $\mathcal{M}\subset \mathbb{R}^d$, simulating the Langevin equation necessitates computationally costly sufficiently small time steps to resolve the fast attraction toward the submanifold \cite{diffusion2}. 

\begin{rem} \label{rem:2}
We point to a connection of (\ref{eq:update_ss}b) to self-attention transformer architectures \cite{Transformer}. We recall that the attention function acts on a matrix $Q\in \mathbb{R}^{N \times d}$ of $N$ queries, a matrix $K \in \mathbb{R}^{M \times d}$ of $M$ keys, and a matrix $V \in \mathbb{R}^{M \times d}$ of $M$ values in the form of
\begin{equation} \label{eq:attention}
\mbox{Attention}\,(Q,K,V) = \mbox{softmax} \left(
\frac{Q K^{\rm T}}{\sqrt{d}} \right) V.
\end{equation}
In the context of (\ref{eq:update_ss}b) we find that $Q = X(n+1/2)^{\rm T}$ and $K = V = \mathcal{X}^{\rm T}$. Hence $N=1$ and the output of the softmax function becomes a probability vector of dimension $1\times M$ which we denote by $\breve{w}$.  Note that $\breve{w} \in \mathbb{R}^{1\times M}$ multiplies $V = \mathcal{X}^{\rm T} \in \mathbb{R}^{M\times d}$ from the left resulting in essentially the transpose of what has been used in (\ref{eq:mean}) with, however, a differently defined probability vector. Indeed, the Schr\"odinger bridge sampler defines the probability vector (\ref{eq:probability vectors}) in a manner closely related to what has been proposed as Sinkformer in \cite{SABP22} and the scaling factor $\sqrt{d}$ in (\ref{eq:attention}) is substituted by $2\epsilon$. More specifically, we note that (\ref{eq:tij}) could be replaced by
\begin{equation} 
(T)_{jk} = \exp \left( \frac{(x^{(k)})^\top x^{(j)}}{2\epsilon}  \right)
\end{equation}
without changing the resulting Schr\"odinger bridge approximation $P$ (albeit with a different scaling vector $v$ compared to (\ref{eq:DM_approximation})). One would then also have to replace (\ref{eq:tvec}) 
and (\ref{eq:probability vectors}) by
\begin{equation} \label{eq:shift}
\hat{t}^{(j)}(x) =  v^{(j)}\exp \left( \frac{x^\top x^{(j)}}{2\epsilon}  \right) =
\exp\left( \frac{x^\top x^{(j)}}{2\epsilon} +
\log v^{(j)}\right)
\end{equation}
and
\begin{equation} 
\hat{w}^{(j)} (x) = \frac{\hat{t}^{(j)}(x)}{\sum_{j=1}^M \hat{t}^{(j)} (x)} 
\end{equation}
for $j=1,\ldots,M$, respectively, in line with self-attention transformer architectures which do not involve the shift by $\log v^{(j)}$ in (\ref{eq:shift}). 
\end{rem}

\noindent
The denoising step (\ref{eq:update_ss}b) has a gradient structure since
\begin{equation} \label{eq:gradient structure}
m(x;\epsilon)  = x + \epsilon \nabla \log p_\epsilon (x)
\end{equation}
with (unnormalised) density
\begin{equation}
p_\epsilon (x) = \left(v^{\top} t(x)\right)^2
\end{equation}
and
\begin{equation}
\nabla \log p_\epsilon (x) = 
-\frac{1}{\epsilon}\frac{\sum_{j=1}^M (x-x^{(j)}) v^{(j)}t^{(j)}(x)}{v^{\rm T} t(x)} = \frac{1}{\epsilon}(\mathcal{X}\, w(x)-x).
\end{equation}
Hence the proposed sampler can be viewed as an EM approximation of the modified Langevin dynamics
\begin{equation}
    \dot{X}(\tau) = \nabla \log p_\epsilon(X(\tau)) + \sqrt{2}\,\dot{W}(\tau).
\end{equation}
A modified score has also been considered in the form of Moreau--Yosida regularised score functions in \cite{DMP18} and smoothed score functions in the form of plug \& play priors in \cite{plugandplay}. Contrary to those approaches, the modified score $\nabla \log p_\epsilon(x)$ arises from the Schr\"odinger bridge approximation of the semi-group $\exp(\epsilon \mathcal{L})$ with generator $\mathcal{L}$ \cite{Pavliotis2016} given by 
\begin{equation}
\mathcal{L}f = \nabla \log \pi(x) \cdot \nabla f + \Delta f.
\end{equation}

%
\subsection{Kernel denoising and Tweedie's formula}
%

We note that (\ref{eq:update_ss}b) is related to MMSE denoising as widely used to reduce random fluctuations in a signal. The connection between score estimation, autoencoders, and denoising has been discussed in \cite{Vincent,AB14}. See also the recent survey \cite{MD24}. However, while MMSE denoising typically considers conditional mean estimators in pseudo-linear form \cite{MD24} or in the form of auto-encoders \cite{AB14}, our approach relies on (nonlinear) conditional mean estimators  of the form (\ref{eq:mean}), which also arise from kernel denoising \cite{MD24}, which is closely related to Tweedie's formula \cite{tweedie} as we explain next. 

Given the data distribution $\pi$ and a scale parameter $\gamma >0$, consider the extended (unnormalized) distribution $\Pi_\gamma$ in $(x,x') \in \mathbb{R}^{2d}$ defined by
\begin{equation}
\Pi_\gamma(x,x') = \exp \left(-\frac{1}{2\gamma}
\|x-x'\|^2 \right) \,\pi(x')
\end{equation}
and its (unnormalized) marginal distribution
\begin{equation}
    \pi_\gamma (x) = \int \Pi_\gamma (x,x')\,{\rm d}x'.
\end{equation}
Tweedie's formula \cite{tweedie} states that
\begin{equation} \label{eq:tweedie}
    \nabla \log \pi_\gamma (x) = -\frac{1}{\gamma} \left(x - \mathbb{E}[x'|x]\right)
\end{equation}
with the conditional expectation value defined by
\begin{equation}
\mathbb{E}[x'|x] = \frac{\int x' \,\Pi_\gamma (x,x')\,{\rm d}x'}{\pi_\gamma (x)}.
\end{equation}
We may now replace the data distribution $\pi_\gamma$ by the empirical measure over the training samples $\{x^{(j)}\}_{j=1}^M$ to obtain the equally weighted (unnormalized) Gaussian kernel density estimator (KDE)
\begin{equation} \label{eq:KDE}
\tilde{\pi}_\gamma (x) = 
\sum_{j=1}^M \exp \left(-\frac{1}{2\gamma}
\|x-x^{(j)}\|^2 \right),
\end{equation}
which, according to (\ref{eq:tweedie}), leads to the score function
\begin{align}    
s(x;\gamma) = \nabla \log \tilde{\pi}_\gamma (x) 
= -\frac{1}{\gamma}\left( x - \frac{\sum_{j=1}^M
x^{(j)} \exp \left(-\frac{1}{2\gamma}
\|x-x^{(j)}\|^2 \right)}{\sum_{j=1}^M \exp\left(-\frac{1}{2\gamma}
\|x-x^{(j)}\|^2 \right)} \right).
\end{align}
Using this score function in (\ref{eq:EM_SDE}) with $\gamma = \epsilon$ results in a scheme of the form (\ref{eq:propagator_general}) with $S(x;\epsilon) = 2\epsilon I$ and the conditional mean estimator $m(x;\epsilon)$ being replaced by the denoiser
\begin{equation} \label{eq:KDE drift}
D(x;\epsilon) := \frac{\sum_{j=1}^M
x^{(j)} \exp \left(-\frac{1}{2\epsilon}
\|x-x^{(j)}\|^2 \right)}{\sum_{j=1}^M \exp\left(-\frac{1}{2\epsilon}
\|x-x^{(j)}\|^2 \right)} = \mathcal{X} \,\tilde w(x),
\end{equation}
where the weight vector $\tilde w(x) \in \mathbb{R}^M$ is now defined by
\begin{equation}
\tilde w^{(j)}(x) = \frac{\exp \left(-\frac{1}{2\epsilon}
\|x-x^{(j)}\|^2 \right)}{\sum_{j=1}^M \exp\left(-\frac{1}{2\epsilon}
\|x-x^{(j)}\|^2 \right)}, \qquad j = 1,\ldots,M.
\end{equation}
We find that $m(x;\epsilon)$ and $D(x;\epsilon)$ differ through the additional Sinkhorn weight vector $v\in \mathbb{R}^M$ in (\ref{eq:probability vectors}) and the  scale parameter $2\gamma=2\epsilon$ in (\ref{eq:KDE}) compared to the scale parameter $4\epsilon$ used in \eqref{eq:tvec}. The connections drawn in Remark \ref{rem:2} to transformer architectures apply equally to (\ref{eq:KDE drift}).

The results of \cite{WR20} suggest that (\ref{eq:mean}) provides a more accurate approximation to the conditional expectation value (\ref{eq:CE}) than (\ref{eq:KDE drift}), which is based  on the equally weighted Gaussian mixture approximation (\ref{eq:KDE}).  In particular, it holds that
\begin{equation}
    m(x;\epsilon) := \exp(\epsilon \mathcal{L}) \,{\rm id}(x) = x + \epsilon \nabla \log \pi(x) + \mathcal{O}(\epsilon^2)
\end{equation}
in the limit $M\to \infty$, while Tweedie's formula formally leads to
\begin{equation}
D(x;\epsilon) := x + \epsilon \nabla \log
\tilde \pi_\epsilon (x) = x + \epsilon \nabla \log \pi(x) + \mathcal{O}(\epsilon^2).
\end{equation}
Here ${\rm id}(x)=x$ denotes the identity map. Hence, to leading order in $\epsilon$, both approaches agree. However, while Tweedie's formula leads to an approximation error that arises from replacing $\pi$ by a regularized density ${\tilde{\pi}}_\epsilon$, the Schr\"odinger bridge sampler leads to higher-order corrections which are consistent with the actual underlying Langevin dynamics. This becomes particularly appealing when implemented together with the data-aware covariance matrix (\ref{eq:cm_estimate}) instead of a constant $S(x;\epsilon) = 2\epsilon I$ in (\ref{eq:propagator_general}) or when a variable bandwidth is implemented in \eqref{eq:tij} as was done in \cite{GLRY24}. A precise statement will be the subject of future research. We also stress that the Schr\"odinger bridge sampler can easily be extended to Langevin dynamics with multiplicative noise \cite{GLRY24} while such an extension is unclear when based on a KDE.

While (\ref{eq:mean}) works well for low-dimensional problems and sufficiently large sample sizes $M$, applications to medium- or high-dimensional problems have remained an open challenge since accurate approximations of the Schr\"odinger bridge problem require an exponentially increasing number of samples as the dimension, $d$, of the sample space $\mathbb{R}^d$ increases (cf.~\cite{WR20}). The curse of dimensionality applies equally to the KDE-based approximation (\ref{eq:KDE drift}) and a failure to generalize has been discussed recently in the context of score-generative models \cite{LCL24}.

The key observation of this paper is that the approximation of conditional expectations (\ref{eq:CE}) via Schr\"odinger bridges does not necessarily require the full Markov chain (\ref{eq:DM_approximation}) and that localization can be applied provided conditional independence can be established. This idea will be developed in the following section. Localization will subsequently be extended to kernel-based denoising in Subsection \ref{sec:local KDE}.

%
\section{Localized Schr\"odinger bridge sampler}
%
\label{sec:lsbs}

To introduce the main idea of localizing the Schr\"odinger bridge sampler developed in \cite{GLRY24} we first consider an illustrative example of sampling from a multivariate Gaussian distribution. We will see that localization allows for a  significant reduction of the number of samples required to achieve a certain accuracy. In particular, the number of samples required to achieve a certain accuracy does not depend on the intrinsic dimension of the samples but rather is determined by the conditional independence which typically leads to a sequence of much lower dimensional estimation problems. 

%
\subsection{Motivational example: Gaussian setting}
\label{sec:Example}
%

Let $\Delta_h \in \mathbb{R}^{d\times d}$ denote the standard discrete Laplacian over a periodic 
domain $[0,L]$ of length $L>0$ with mesh-size $h = L/d$. We assume that the sampling distribution $\pi(x)$ is
Gaussian with zero mean and covariance matrix
\begin{equation}
C = (I - \Delta_h)^{-1}.
\label{eq:C}
\end{equation}
Instead of the distribution $\pi(x)$, we are given $M$ samples $x^{(j)} \sim {\rm N}(0,C)$, $j=1,\ldots,M$, and denote their $\loc$-th entry by $x_\loc^{(j)}$ for $\loc=1,\ldots,d$. The goal is to produce more samples from ${\rm N}(0,C)$ using the time-stepping scheme (\ref{eq:update_ss}) without making explicit reference to the unknown covariance matrix $C$. This particular setting of a generative model can become arbitrarily challenging by either increasing $L$ for fixed mesh-size $h$ or by decreasing the mesh-size $h = L/d$ for fixed $L$.

In order to gain some insight into the problem, we first consider the standard EM sampler
in case the distribution is known; i.e.,
\begin{equation} \label{eq:EM sampler}
    X(n+1) = X(n) - \epsilon( I - \Delta_h) X(n) +
    \sqrt{2\epsilon}\, \Xi(n), \qquad \Xi(n) \sim {\rm N}(0,I).
\end{equation}
Because of the structure of $\Delta_h$, we can rewrite the EM update in the components of $X(n)$ 
in the form
\begin{equation}
X_\loc(n+1) = w_{-1} \,X_{\loc-1}(n) 
+ w_{0}\, X_\loc(n) + w_{1} \,X_{\loc+1}(n) 
+ \sqrt{2\epsilon}\,\Xi_\loc (n), \qquad \loc = 1,\ldots,d,
\label{eq:Xdel}
\end{equation}
with weights
\begin{equation}
w_{\pm 1} = \frac{\epsilon}{2h^2}, \qquad 
w_{0} = 1 -\epsilon\left(1 + \frac{1}{h^2}\right)
\label{eq:wlap}
\end{equation}
and periodic extension of $X_\alpha$ for
$\alpha = 0$ and $\alpha = d+1$. We assume that the step size $\epsilon$ is chosen such that $w_0 \ge 0$.
The EM update \eqref{eq:Xdel} reveals that the conditional expectation value of $X_\loc (n+1)$ only depends on the value of the neighboring grid points of $X(n)$ with weights $w_0$ and $w_{\pm 1}$;
i.e., 
\begin{subequations} 
\label{eq:conditional expectation localized}
\begin{align}
\mathbb{E}[X_\loc(n+1)\,|\,X(n)] &=
\mathbb{E}[X_\loc(n+1)\,|\, (X_{\loc-1}(n),X_\loc(n),X_{\loc+1}(n))] \\
&= w_{-1} \,X_{\loc-1}(n) 
+ w_{0}\, X_\loc(n) + w_{1} \,X_{\loc+1}(n).
\end{align}
\end{subequations}
It is convenient to introduce the short-hand
\begin{equation}
    X_{\locind} := (X_{\loc-1},X_\loc,X_{\loc+1})^{\rm T} \in \mathbb{R}^{d_\loc}, 
    \label{eq:xloc}
\end{equation}
with $d_\loc = 3$, to denote the set of neighboring grid points of $X_\loc$. 

To help the reader navigating the various indices and sub- and superscripts we summarize here our notation. Superscripts $(j)$ are reserved to denote samples $j=1,\ldots,M$ as well as components of vectors in $\mathbb{R}^M$. For example, the components of the probability vector $w \in \mathbb{R}^M$ are denoted by $w^{(j)}$.
The Greek subscript $\loc $ with $\loc=1,\ldots,d$ is reserved to denote components of a vector $x$ in state space $\mathbb{R}^d$, i.e.~$x_\loc$ for $\loc = 1,\ldots,d$. Subscripts $\locind$ are reserved to denote localization around a component $\loc$; i.e., $x_\locind \in \mathbb{R}^{d_\loc}$.

The dependency of the conditional expectation value \eqref{eq:conditional expectation localized} on the neighboring points is to be exploited in the update step (\ref{eq:update_ss}b), which we recall here in its component-wise formulation as
\begin{equation}
X_\loc(n+1) = \sum_{j=1}^M x_\loc^{(j)}w^{(j)} (X(n+1/2)) ,
\end{equation}
for $\alpha=1,\ldots,d$.
We recall from our previous considerations that the conditional expectation value of $X_\loc(n+1)$ should depend on $X_{\locind}(n+1/2)$ only. Hence the question arises whether we can find appropriately localized probability vectors $w(x)$ for the Schr\"odinger bridge sampler (\ref{eq:update_ss}).  The following formal argument can be made. We restrict ${\rm N}(0,C)$ to $x_\locind \in \mathbb{R}^{d_\loc}$ and truncate the samples $x^{(j)}$, $j=1,\ldots,M$, accordingly to yield $x_{\locind}^{(j)}$. The covariance matrix $C_{\rm r} \in \mathbb{R}^{d_\loc\times d_\loc}$ of the reduced random variables $X_{\locind} \in \mathbb{R}^{d_\loc}$ is simply the restriction of $C$ to the corresponding sub-space, which in this particular example is independent of $\alpha$. Furthermore, using the Schur complement, one finds
\begin{equation}
    C_{\rm r}^{-1} = \left( \begin{array}{ccc} \ast  & -\frac{1}{2h^2} & \ast\\
    -\frac{1}{2h^2} & 1+\frac{1}{h^2} & -\frac{1}{2h^2} \\ \ast & -\frac{1}{2h^2} & \ast
    \end{array} \right),
\end{equation}
where $\ast$ denotes entries which differ from the matrix which would be obtained by restricting $C^{-1}$ to the corresponding sub-space. The important point is that the central 
elements remain identical (cf. \eqref{eq:wlap}) and that only those entries enter the approximation of the conditional expectation value (\ref{eq:conditional expectation localized}).

We now describe a localized Schr\"odinger bridge approach for this specific problem. One replaces the matrix $T \in \mathbb{R}
^{M\times M}$ with entries (\ref{eq:tij})
by localized matrices $T_\loc \in \mathbb{R}^{M\times M}$ with entries
\begin{equation} \label{eq:tij_local}
(T_\loc)_{jk} = \exp \left( -\frac{1}{4 \epsilon} \|x_\locind^{(j)}-x_\locind^{(k)}\|^2 
\right), \qquad 
j,k = 1,\ldots,M,
\end{equation}
for fixed $\loc \in \{1,\ldots,d\}$. For each of these localized matrices $T_\alpha$ we employ the local Sinkhorn algorithm to obtain the Sinkhorn weights $v_\loc \in \mathbb{R}^M$ for $\loc=1,\ldots,d$, which render
\begin{align}
P_\loc = D(v_{\loc})T_\loc D(v_{\loc})
\end{align}
bistochastic (cf.\eqref{eq:PDTD}). The key point is that the Euclidean norm in $\mathbb{R}^d$, $d\gg 1$, is replaced by the Euclidean norm in $\mathbb{R}^{d_\loc}$ with $d_\loc = 3$. Furthermore, in this particular example, the corresponding Schr\"odinger bridge approximately couples the restricted Gaussian distribution ${\rm N}(0,C_r)$ with itself. Next, the single $M$-dimensional probability vector (\ref{eq:probability vectors}) is replaced by $d$ $M$-dimensional probability vectors
\begin{equation} \label{eq:pv local}
    w_\loc(x_\locind) := \frac{D(v_\loc) \, t_\loc(x_\locind)}
    {v_\loc^{\top} t_\loc (x_\locind)}, \qquad \loc =1,\ldots,d,
\end{equation}
which depend on $x_\locind \in \mathbb{R}^{d_\loc}$ and where the vector-valued function $t_\loc (x_\locind) \in \mathbb{R}^M$ has entries
\begin{equation}
\label{eq:tij_local2}
t_\loc^{(j)}(x_\locind) =  \exp \left( -\frac{1}{4\epsilon} \|x_\locind^{(j)}-x_\locind\|^2 
 \right), \qquad 
j = 1,\ldots,M.
\end{equation}
Note that (\ref{eq:pv local}) depends on the restricted vectors $x_\locind^{(j)} \in \mathbb{R}^{d_\loc}$, $j=1,\ldots,M$, only. It can be verified by explicit calculation that the interpolation property
\begin{equation}
    w^{(j)}_\loc (x^{(k)}_\locind) = (P_\loc)_{jk}, \qquad j,k = 1,\ldots,M,
\end{equation}
holds.

We obtain the localized approximation 
\begin{equation}
\label{eq:mloc}
m_\loc(x_\locind;\epsilon) = \mathcal{X}_{\loc} \,w_\loc (x_\locind), \qquad \loc = 1,\ldots,d,
\end{equation}
of the conditional expectation values, where
\begin{equation}
\mathcal{X}_{\loc} = (x_{\loc}^{(1)},\ldots,x_{\loc}^{(M)}) \in \mathbb{R}^{1\times M}.
\end{equation}
We also introduce the localized data matrix
\begin{equation}
\mathcal{X}_\locind = (x_{\locind}^{(1)},\ldots,x_{\locind}^{(M)}) \in \mathbb{R}^{d_\loc\times M},
\end{equation}
which enters into the computation of $T_\alpha$. 

For constant diffusion $S(x;\epsilon) = 2\epsilon I$ the localized variant of the iteration scheme \eqref{eq:propagator_general} becomes 
\begin{equation} \label{eq:propagator_general_loc}
X_\loc(n+1) = m_\loc(X_\locind(n);\epsilon) + \sqrt{2 \epsilon} \,\Xi_\loc(n), \qquad \loc = 1,\ldots,d,
\end{equation}
for $\Xi(n) \sim {\rm N}(0,I)$. Similarly, the split-step scheme \eqref{eq:update_ss} becomes 
\begin{subequations} \label{eq:update_ss_loc}
    \begin{align}
        X_\locind(n+1/2) &= X_\locind(n) + \sqrt{2 \epsilon}\, \Xi_\locind(n),\\
        X_\loc(n+1) &=  m_\loc (X_\locind(n+1/2);\epsilon).
    \end{align}
\end{subequations}
In other words, we have replaced a single Schr\"odinger bridge update in $\mathbb{R}^d$ by $d$ Schr\"odinger bridge updates in $\mathbb{R}^{d_\loc}$.

\begin{rem} \label{rem:3}
In line with the discussion on transformer architectures from Remark \ref{rem:2}, we wish to point to a connection to multi-head attention \cite{Transformer}. More specifically, our localization procedure has introduced $d$ heads each relying on $\mathcal{X}_\locind$ as matrix of key vectors, $\mathcal{X}_\loc$ as matrix of value vectors, and $X_\locind (n+1/2)$ as query vector in order to produce an update in the scalar-valued entries  $X_\loc(n+1/2)$ for $\alpha = 1,\ldots,d$.
\end{rem}

\noindent
We finally discuss a localized version of the data-aware covariance matrix (\ref{eq:cm_estimate}). Given 
localized weights $w_\loc (x_\locind)$ and localized data matrices $\mathcal{X}_\locind$, we define the
$d_\loc \times d_\loc$-dimensional covariance matrices
\begin{equation} \label{eq:cm_estimate_local}
S_\alpha (x_\locind;\epsilon) = \mathcal{X}_\locind D(w_\loc(x_\locind)\mathcal{X}_\locind^\top  \,
 -
 \mathcal{X}_\locind w_\loc(x_\locind) \, w_\loc(x_\locind)^\top \mathcal{X}_\locind^\top
\end{equation}
for $\alpha = 1,\ldots,d$. Given a sample $\Xi(n) \sim {\rm N}(0,I)$, one first computes the $d_\loc$-dimensional vector $\sqrt{S_\alpha (x_\locind;\epsilon)}\,\Xi_\locind(n)$ of which one picks its scalar entry corresponding to $x_\loc$, which we denote by $\xi_\loc (n)$. The localized variant of (\ref{eq:propagator_general}) then becomes
\begin{equation} \label{eq:propagator general local}
X_\loc(n+1) = m_\loc(X_\locind(n);\epsilon) + \xi_\loc(n), \qquad \loc = 1,\ldots,d.
\end{equation}

%
\begin{figure}[!htp]
    \centering $\,$
    \includegraphics[scale=0.27]{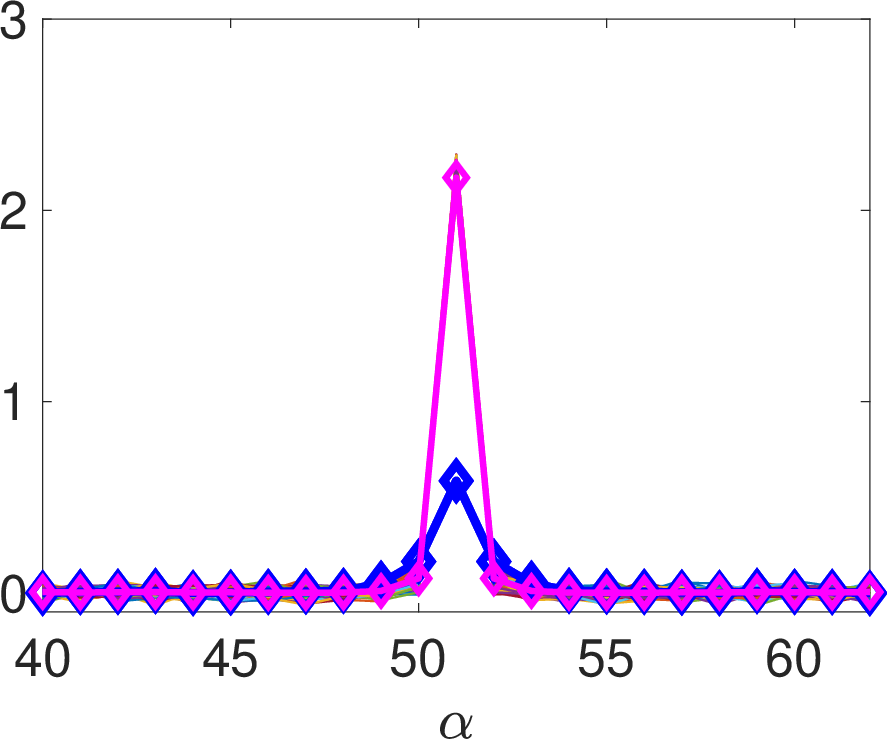}$\quad\,\,\,$
    \includegraphics[scale=0.27]{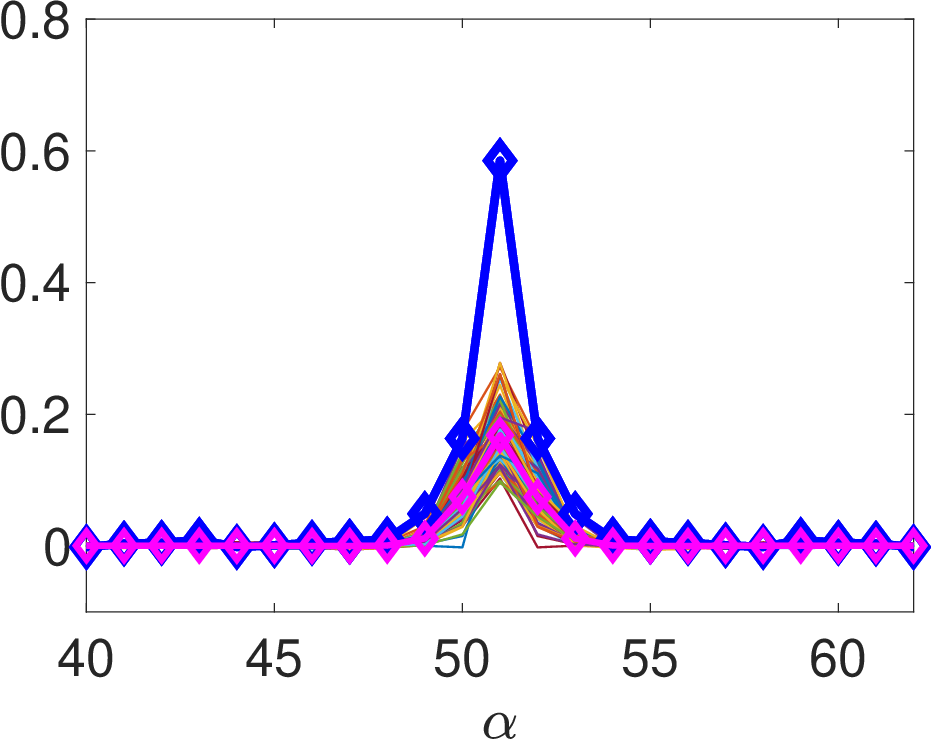}$\quad\,\,$
    \includegraphics[scale=0.27]{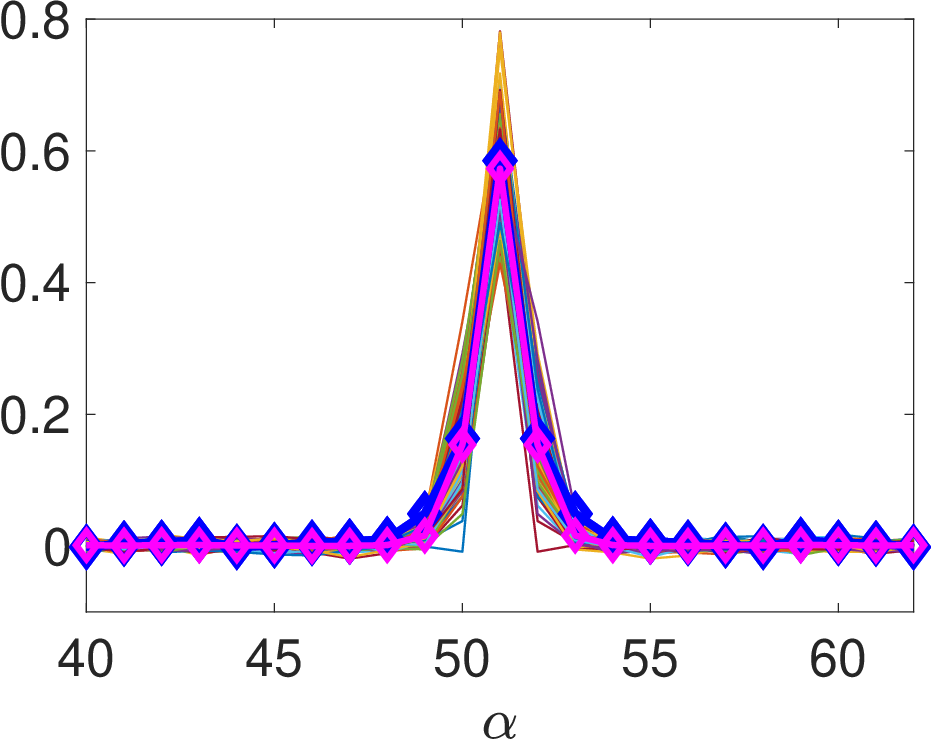}
    \\
    \includegraphics[scale=0.26]{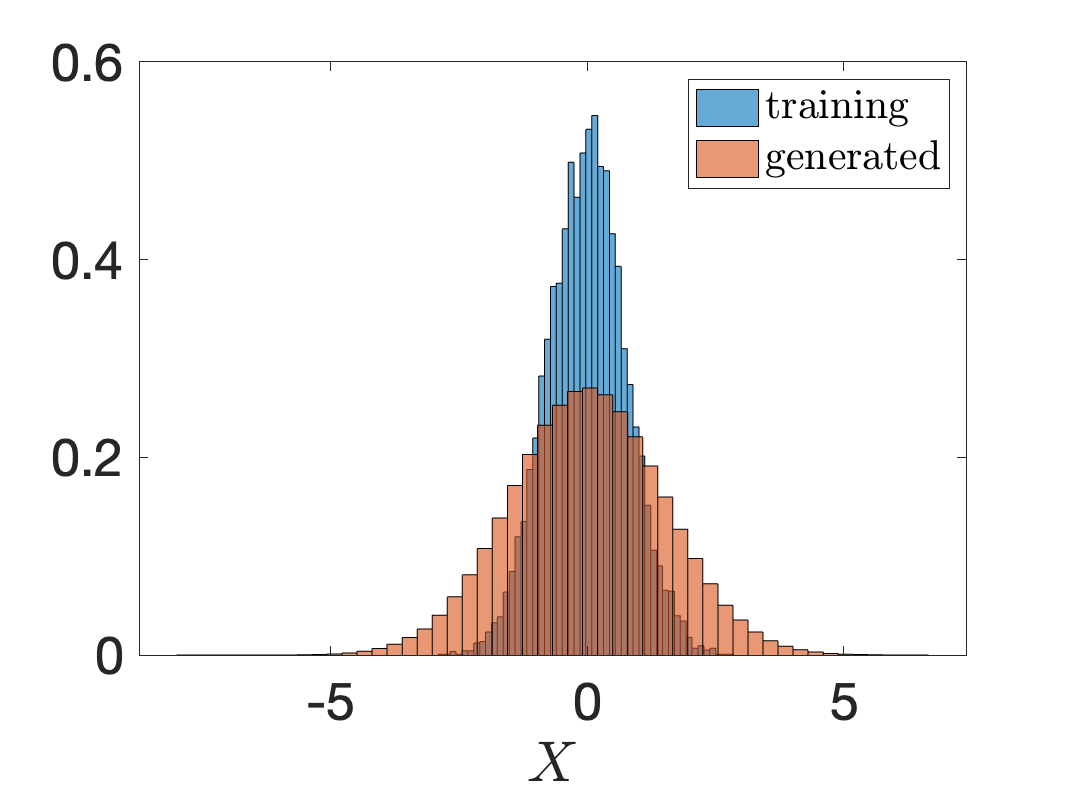}
    \includegraphics[scale=0.26]{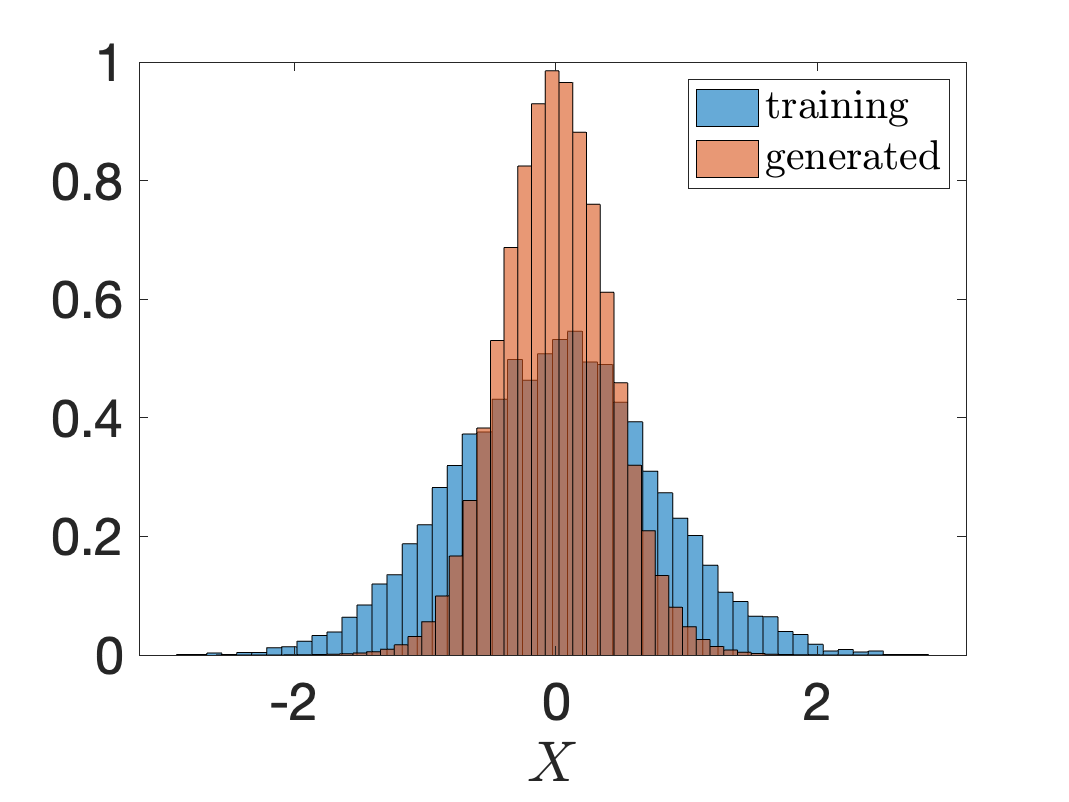}
    \includegraphics[scale=0.26]{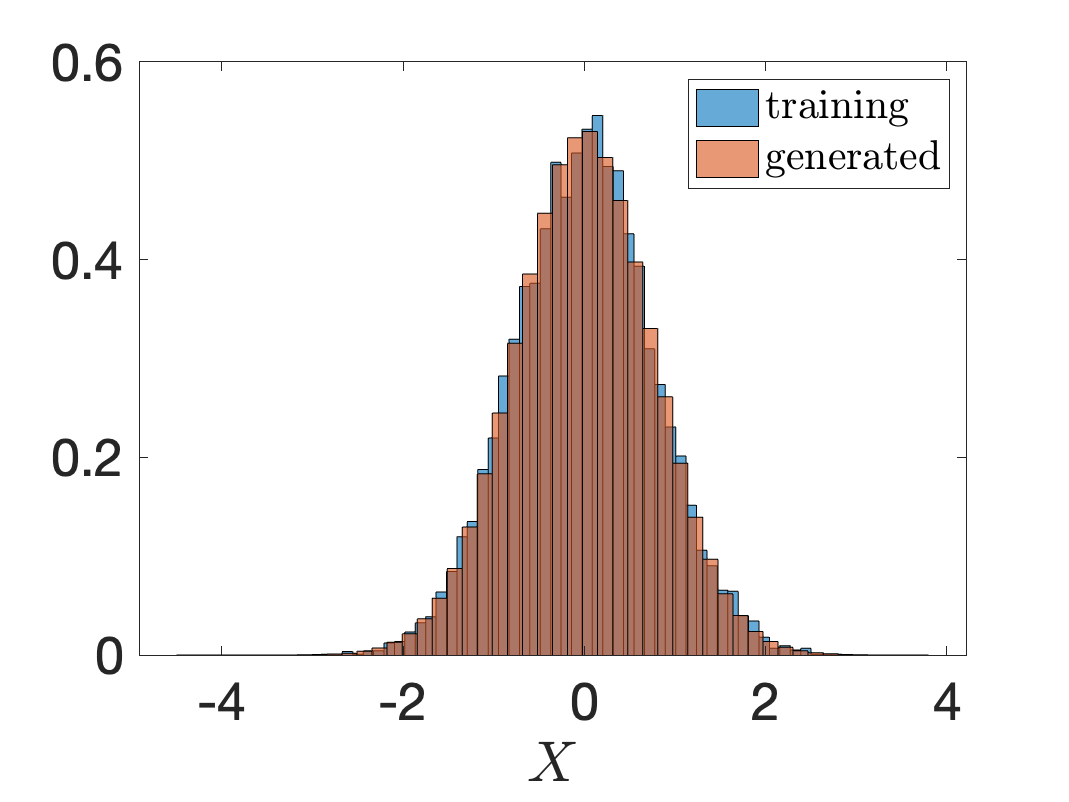}
    \caption{Comparison of the samples obtained from three different variants of localized Schr\"odinger bridge samplers. We show the centered rows of the empirical covariance matrix $\hat C$ (top row) and empirical histograms (bottom row) obtained from using all $d=101$ components. The blue markers denote the empirical covariance for the given samples; the magenta markers show the average over all $d$ rows. Left column: Localized EM-type sampler \eqref{eq:propagator_general_loc}; middle column: Localized split-step sampler \eqref{eq:update_ss_loc}; right column: Localized sampler (\ref{eq:propagator general local}) with data-aware diffusion matrix \eqref{eq:cm_estimate_local}. Given the large value of $\epsilon=1$, only (\ref{eq:propagator general local}) is able to faithfully reproduce the target measure ${\rm N}(0,C)$.}   
    \label{fig:multiGauss1}
\end{figure}
%

%
\subsubsection{Numerical illustration} \label{sec:Gauss}
%
To illustrate how well the localized sampling strategy is able to generate samples from a multivariate Gaussian, we generate $M$ training samples of a $d$-dimensional multivariate Gaussian with
\begin{align}
x^{(j)}\sim {\rm N}(0,C)
\end{align}
for $j=1,\ldots,M$ with a $d\times d$ covariance matrix of the form \eqref{eq:Sigma} with a tridiagonal precision matrix with $C^{-1}_{i,i}=2$, $C^{-1}_{i,i\pm 1}=-0.5$, and periodic conditions $C^{-1}_{1,d}=C^{-1}_{d,1}=-0.5$. The corresponding entries of the covariance matrix are $C_{i,i} \approx 0.58$, $C_{i,i\pm1} \approx 0.15$, $C_{i,i\pm 2} \approx 0.04$, and $C_{i,i\pm 3} \approx 0.01$.

We employ three different implementations of the localized Schr\"odinger bridge sampler with a localization set comprised of two neighboring grid points, i.e.~$d_\loc = 3$. These implementations are (i) the split-step scheme (\ref{eq:update_ss_loc}), (ii) the localized EM-type scheme \eqref{eq:propagator_general_loc}, and (iii) the scheme (\ref{eq:propagator general local}) with data-aware noise update.

In Figure~\ref{fig:multiGauss1} we compare the generated new samples with the given samples for all three sampling strategies. We show the resulting empirical histograms as well as the rows of the empirical covariance matrix
$\hat C$. The rows are centered at the middle point using periodicity. We use $M=100$ training samples and a bandwidth of $\epsilon=1$ for $d=101$ to generate $10,000$ new samples. 

While the split-step scheme  \eqref{eq:update_ss_loc} underestimates the variance of the distribution, the EM-type scheme \eqref{eq:propagator_general_loc} overestimates the variance. Only the localized scheme with data-aware diffusion (\ref{eq:propagator general local}) captures the marginal distribution and the covariance structure with desirable accuracy.

In Figure~\ref{fig:multiGauss2} we investigate the behavior of the localized split-step scheme \eqref{eq:update_ss_loc} for varying parameter $\epsilon \in \{0.01,0.1,1\}$. We find that $\epsilon = 0.1$ improves the results while a further decrease to $\epsilon = 0.01$ leads to results comparable to $\epsilon = 1.0$. Similar results are obtained for the EM-type sampling scheme with constant diffusion \eqref{eq:propagator_general_loc}. An optimal choice of $\epsilon$ depends, of course, on the number of available training samples, which has been set to $M=100$ for these experiments.

We remark that unlocalized Schr\"odinger bridge samplers generate samples with a rather noisy correlation structure which is to be expected for $M = 100$ training samples. 

\begin{figure}[!htp]
    \centering
    \includegraphics[scale=0.27]{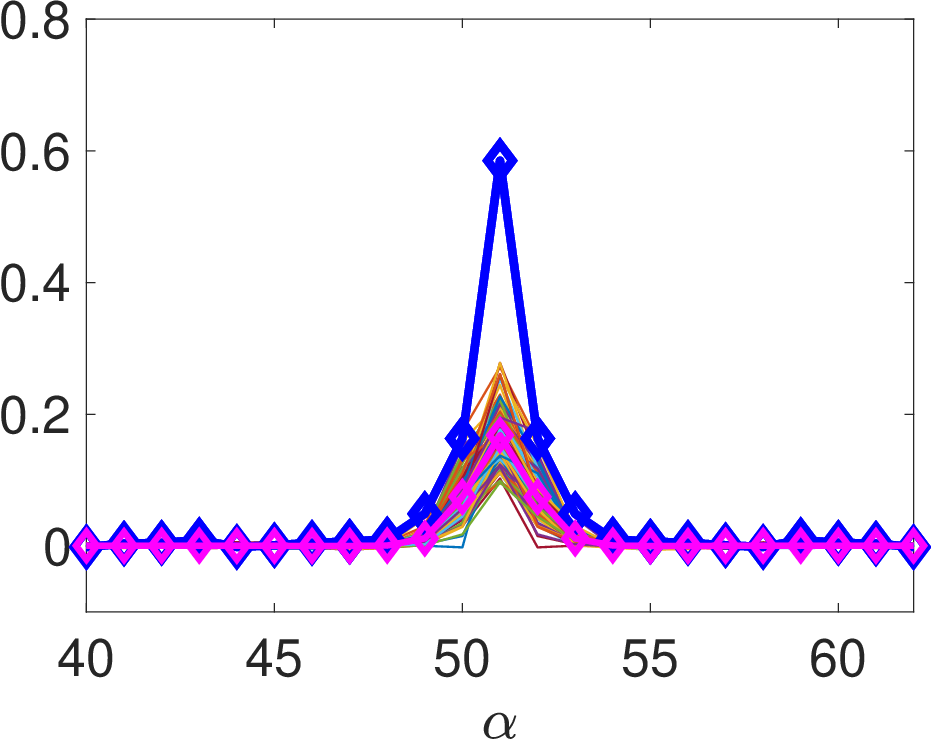}$\quad\,\,\,$
    \includegraphics[scale=0.27]{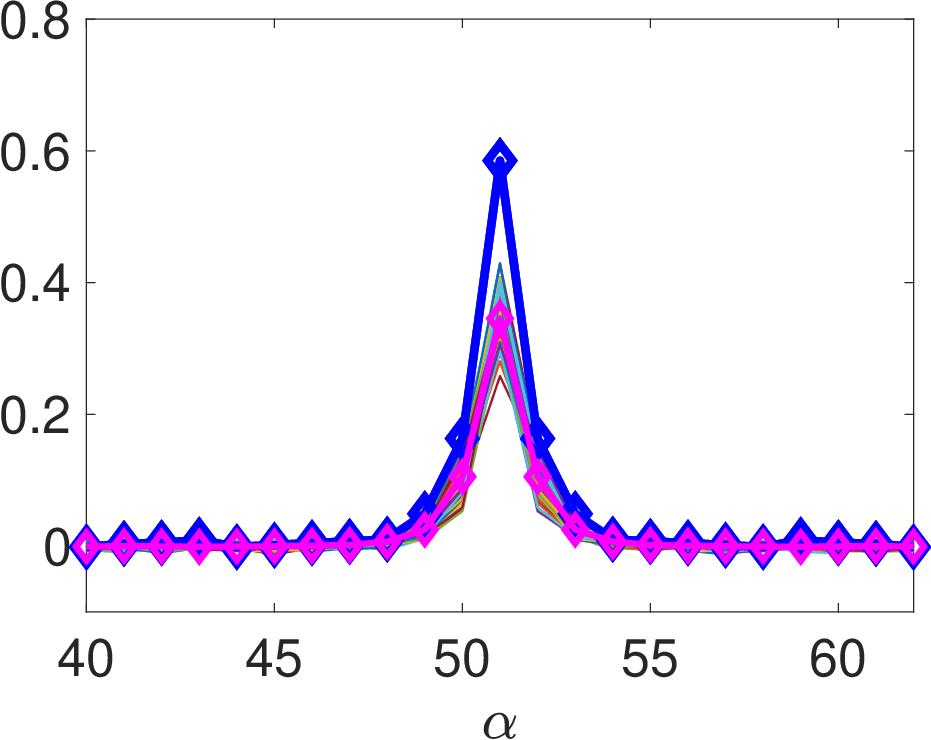}$\quad\,\,\,$
    \includegraphics[scale=0.27]{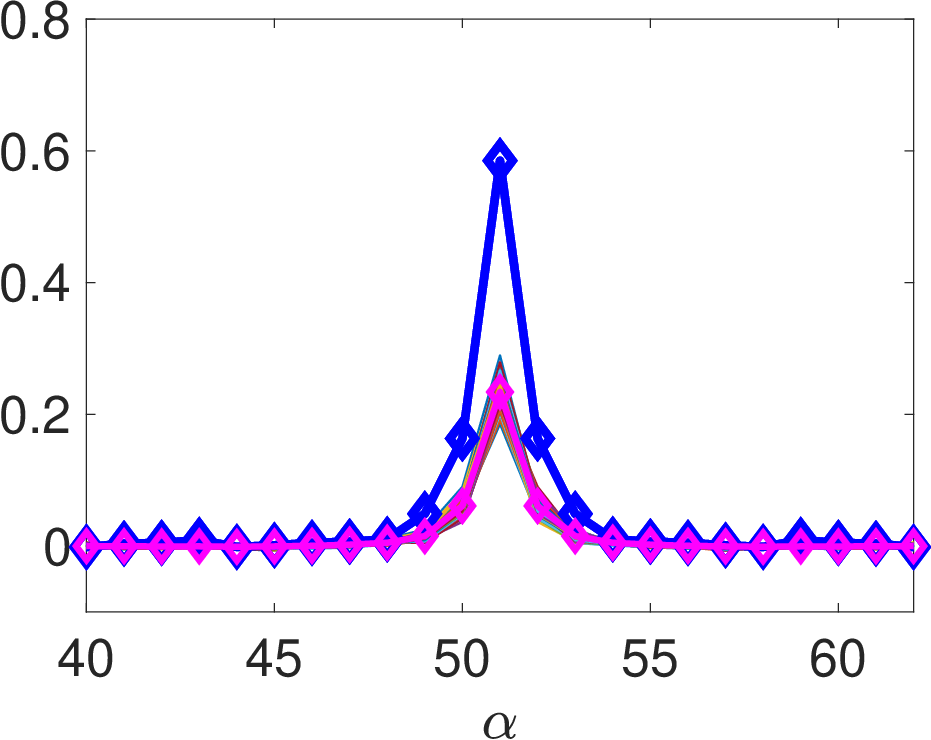}
    \\
    \includegraphics[scale=0.26]{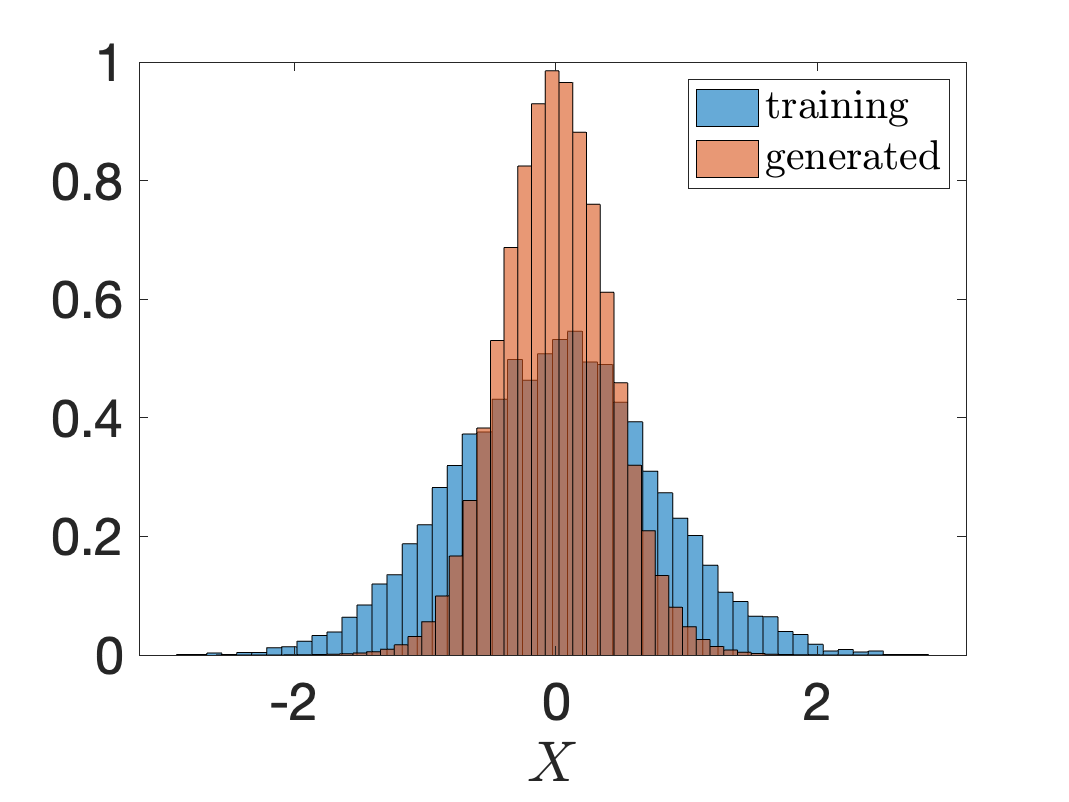}
    \includegraphics[scale=0.26]{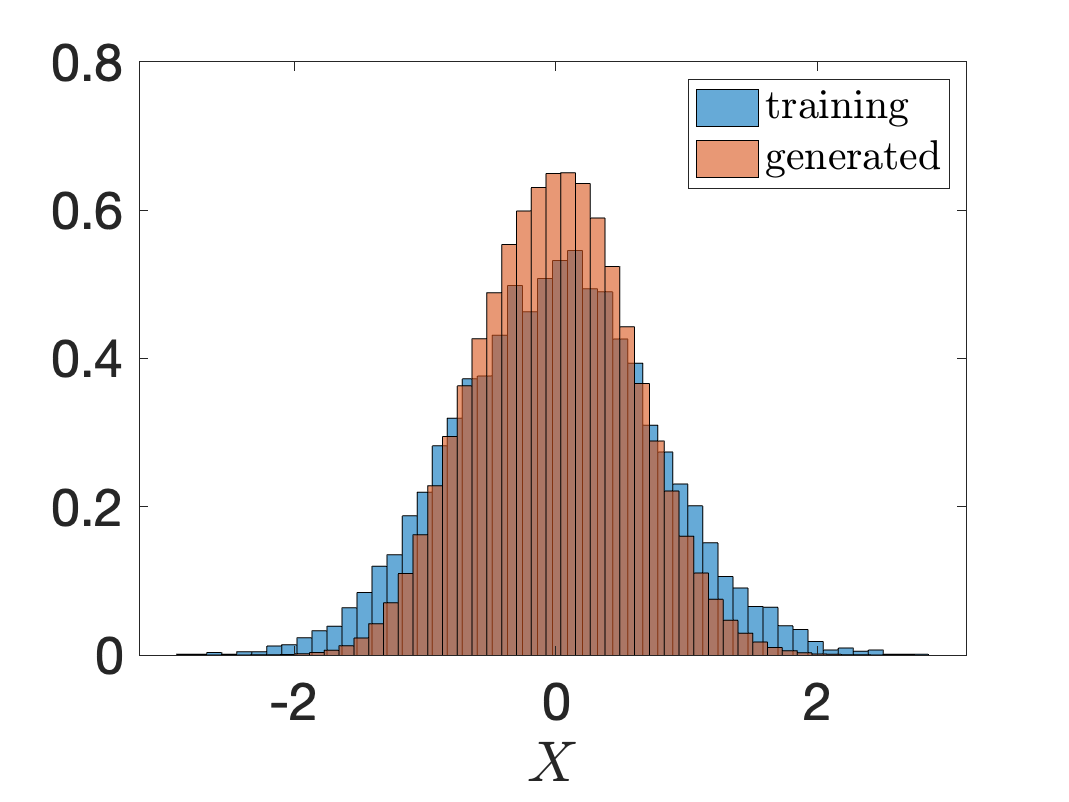}
    \includegraphics[scale=0.26]{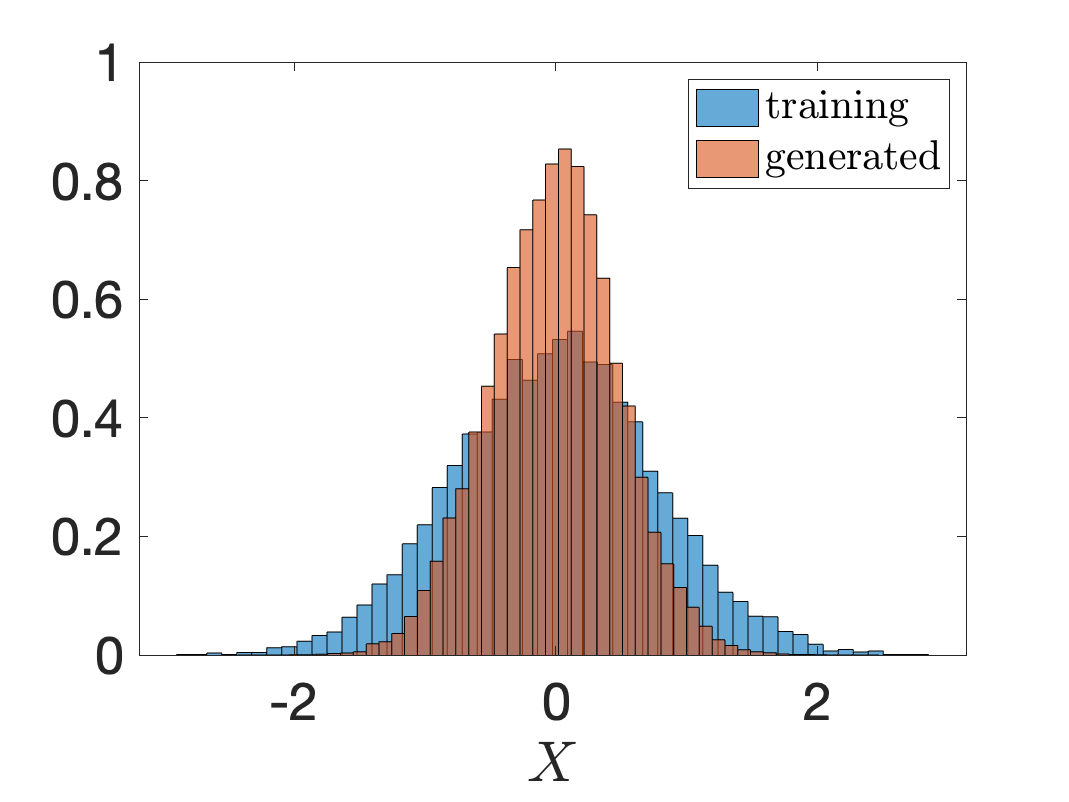}
    \caption{Comparison of the samples obtained from the localized split-step sampler \eqref{eq:update_ss_loc} for varying parameter $\epsilon$. Left column: $\epsilon = 1.0$; middle column: $\epsilon = 0.1$; right column: $\epsilon = 0.01$. While $\epsilon = 0.1$ leads to improved results, it is found that larger and smaller values of $\epsilon$ degrade the performance of the split-step sampler.}   
    \label{fig:multiGauss2}
\end{figure}
%

%
\subsection{Localized Schr\"odinger bridge sampler for general measures} 
%

The strategy of constructing a dimension-reduced localized Schr\"odinger bridge sampler as presented in the previous example of a multivariate Gaussian readily extends to general target measures $\nu({\rm d}x)$. 

We need to introduce some further notation. For each $\loc$-th entry in the state vector $x \in \mathbb{R}^d$ we introduce a subset $\Lambda (\loc) \subset \{1,\ldots,d\}$ and the associated restriction $x_{\locind} \in \mathbb{R}^{d_\loc}$ of $x \in \mathbb{R}^d$ of dimension $d_\loc = \mbox{card}\,(\Lambda (\loc))$. The complementary part of the state vector is denoted by $x_{\backslash \locind} \in \mathbb{R}^{d-d_\loc}$. In the example of the multivariate Gaussian introduced in Section~\ref{sec:Example}, we have $\Lambda(\loc)=\{\loc-1,\loc,\loc+1\}$ with the obvious periodicity extensions for $\alpha=1$ and $\alpha=d$. With this notation in place, the implementation of the localized Schr\"odinger bridge sampler proceeds as described in Section \ref{sec:Example}. 

The key assumption we make is that of conditional independence of $x_\loc$ on $x_{\backslash \locind}$, which allows for the dimension reduction. We consider variables $x_\loc'$ to be conditionally independent of $x_{\backslash \locind}$ if the conditional distribution in $x_\loc'$, $p_\loc(x_\loc|x;\epsilon)$, satisfies
\begin{align}
p_\loc(x_\loc'|x;\epsilon) = p_\loc(x_\loc'|x_\locind;\epsilon),
\end{align}
which implies $\mathbb{E}[x_\loc^\prime|x] = \mathbb{E}[x_\loc^\prime|x_\locind]$. The conditional expectation value \eqref{eq:mloc} turns out to be a Monte-Carlo approximation of the conditional expectation under this assumption. More precisely, given the transition density of overdamped Langevin dynamics, denoted here by $p(x'|x;\epsilon)$, we obtain
\begin{subequations} \label{eq:conditional mean 1}
\begin{align}
\mathbb{E}[X_\loc(\epsilon)|X(0)=x] &=
\int x_\loc' \,p(x'|x;\epsilon) \,{\rm d}x' =
\int x_\loc' \,p_\loc(x_\loc'|x;\epsilon) \,{\rm d}x_\loc'\\
&=\int x_\loc' \,p_\loc(x_\loc'|x_\locind;\epsilon) \,{\rm d}x_\loc' = \mathbb{E}[X_\loc(\epsilon)|X_\locind(0)=x_\locind],
\end{align}
\end{subequations}
where the second line follows from the conditional independence assumption. 
It is reasonable to assume that we can construct a reversible overdamped Langevin process with invariant distribution 
$\pi_\loc(x_\locind)$ and transition kernel $p_\loc(x_\locind'|x_\locind;\epsilon)$ on $\mathbb{R}^{d_\loc}$. Here $\pi_\loc(x_\locind)$ denotes the marginal distribution of $\pi(x)$ in $x_\locind$. Then detailed balance of this dimension-reduced Langevin process is given by 
\begin{equation}
    p_\loc(x_\locind'|x_\locind;\epsilon)\, \pi_\loc(x_\locind) = p_\loc (x_\locind|x_\locind';\epsilon)\, \pi_\loc (x_\locind').
\end{equation}
Note that in our localized Schr\"odinger bridge sampler detailed balance is ensured by the Sinkhorn algorithm which renders the Markov chain reversible. Detailed balance then implies
\begin{subequations} \label{eq:conditional mean}
\begin{align}
\mathbb{E}[X_\loc(\epsilon)|X_\locind(0)=x_\locind] &=
\int x_\loc' \,p_\loc(x_\locind'|x_\locind;\epsilon) \,{\rm d}x_\locind' \\ 
&=\int x_\loc' \,\frac{p_\loc(x_\locind|x_\locind';\epsilon)}{\pi_\loc(x_\locind)} \,\pi_\loc(x_\locind')\,{\rm d}x_\locind'\\
&= \int x_\loc' \,\rho_\loc (x_\locind'|x_\locind;\epsilon)\,\pi(x_\locind')\,{\rm d}x_\locind' ,
\end{align}
\end{subequations}
with
\begin{equation}
\rho_\loc(x_\locind'|x_\locind;\epsilon) := \frac{p_\loc(x_\locind|x_\locind';\epsilon)}{\pi_\loc(x_\locind)}.
\end{equation}
We note that $\rho_\loc(x_\locind'|x_\locind;\epsilon)$ is a density with respect to the reference measure induced by $\pi_\loc(x_\locind)$. We finally approximate the integral in (\ref{eq:conditional mean}c) via Monte Carlo approximation using the restricted data samples $x_\locind^{(j)} \sim \pi_\loc(x_\locind)$, $j=1,\ldots,M$; i.e.,
\begin{equation}
\mathbb{E}[X_\loc(\epsilon)|X_\locind(0)=x_\locind] 
\approx
\sum_{j=1}^M x_\loc^{(j)} \,w_{\loc}^{(j)}(x_\locind), \qquad x_\loc^{(j)} \sim \pi_\loc,
\end{equation}
with $w_{\loc}^{(j)}(x_\locind) \propto \rho_\loc(x_\locind^{(j)}|x_\locind;\epsilon)$ such that $\sum_j w_{\loc}^{(j)}(x) = 1$. This expression is of the form used in the localized Schr\"odinger bridge sampler \eqref{eq:mloc}. 
Furthermore, since the transition kernels $\rho_\loc(x_\locind'|x_\locind;\epsilon)$ are not available in general, the weight vector $w_\loc(x_\locind) \in \mathbb{R}^{d_\loc}$ is approximated by the Schr\"odinger bridge approach as in \eqref{eq:pv local}. Algorithm~\ref{alg:LSBS} summarizes the localized Schr\"odinger bridge split-step sampler (\ref{eq:update_ss_loc}). The algorithm naturally extends to the sampler (\ref{eq:propagator general local}) with localized data-aware covariance matrix (\ref{eq:cm_estimate_local}).

\begin{algorithm}[hbt!]
\caption{Localized Schr\"odinger bridge sampler}
\label{alg:LSBS}
\KwIn{Samples $\mathcal{X} \in \mathbf{R}^{d\times M}$.}
\Parameter{Bandwidth $\epsilon$. Localization dimension $d_\loc$. Desired number of new samples $N$. Number of decorrelation steps $n_c$.}
\KwOut{New samples $x_{\rm s}^{(j)}$ for $j=1,\ldots,N$.}
\BlankLine
Step 1: Construct transition Sinkhorn weights $v_\locind$\\
\For{$\loc\leftarrow 1$ \KwTo $d$}{
{\rm{construct localized data}} $\mathcal{X}_\locind$\;
construct kernel matrix $T_\loc \in \mathbf{R}^{M\times M}$ from localized data\;
construct Sinkhorn weights $v_{\loc}$ from $T_\loc$\;
}
\BlankLine
Step 2: Generate $N$ new samples $x^{(j)}_{\rm s}$ using the Sinkhorn weights $v_\loc$\\
\For{$j\leftarrow 1$ \KwTo $N$}{
each new sample is started from a random initial sample\\ 
$X(0) \gets x^{(j^\star)}$ for $1\le j^\star \le M$ and random $j^\star$\;
\For{$n\leftarrow 0$ \KwTo $n_c$}{
$\Xi (n) \sim {\rm N}(0,I)$\;
\For{$\loc\leftarrow 1$ \KwTo $d$}{
$X_\locind (n) \gets X(n)$ \;
$\Xi_\locind (n) \gets \Xi(n)$ \;
$X_\locind(n+1/2) = X_\locind(n) +\sqrt{2\epsilon}\,\Xi_\locind (n)$
\Comment*[r]{noising step}
construct vector $t_\loc(X_\locind(n+1/2)) \in \mathbf{R}^M$\;
construct conditional probability $w_\loc(X_\locind(n+1/2)) \in \mathbf{R}^M$ using $v_\loc$\,\;
{\rm{construct localized data}} $\mathcal{X}_\loc$\,\;
$X_\loc(n+1) =\mathcal{X}_\loc \,w_\loc (X_\locind(n+1/2))$
\Comment*[r]{projection step}
}
}
$x^{(j)}_{\rm s} \gets X(n_c)$;
}
\end{algorithm}

\begin{rem}
    We have assumed here a strong form of conditional independence by requesting that
    (\ref{eq:conditional mean 1}) holds for all $\epsilon >0$. In general, such a condition will be satisfied approximately for sufficiently small $\epsilon$ only. Compare the EM sampler (\ref{eq:EM sampler}), which
    provides an accurate approximation to the true transition densities $p (x'|x;\epsilon)$ of the underlying diffusion process for $\epsilon>0$ sufficiently small by ignoring higher-order dependencies. In practice, this requires a careful choice of the dependency set $\Lambda(\loc)$ which defines $x_\locind \in \mathbb{R}^{d_\loc}$.
\end{rem}

\noindent

%
\subsection{Algorithmic properties} \label{sec:properties}
%

We briefly discuss a few important results on the stability and ergodicity of the proposed localized Langevin samplers, which they essentially inherit from the unlocalized Schr\"odinger bridge sampler \citep{GLRY24}.

The following lemma establishes that, since each $w_\loc(x_\locind)$, $\loc = 1,\ldots,d$, is a probability vector for any $\epsilon >0$, the localized update step (\ref{eq:update_ss_loc}b) is stable. In order to simplify notations,
we denote by $m_{\rm loc}(x;\epsilon) \in \mathbb{R}^d$ the vector of localized expectation values with components $m_\loc(x_\locind)$, $\loc = 1,\ldots,d$, defined by (\ref{eq:mloc}). 

\begin{lemma} \label{lemma1} 
Let us introduce the set $\mathcal{C}_M\subset \mathbb{R}^d$ defined by
\begin{equation}
    \mathcal{C}_M = \{x\in \mathbb{R}^d:
    |x_\loc| \le |\mathcal{X}_\loc|_\infty\}.
\end{equation}
It holds that the vector $m_{\rm loc}(x;\epsilon) \in \mathbb{R}^d$ of localized expectation value satisfies
\begin{equation} \label{eq:stability}
    m_{\rm loc}(x;\epsilon) \in \mathcal{C}_M
\end{equation}
for all choices of $\epsilon > 0$ and all $x \in \mathbb{R}^d$.
\end{lemma}

\begin{proof}
    The lemma follows from the fact that the $\loc$-component of $m_{\rm loc}(x;\epsilon)$ is given by~\eqref{eq:mloc} and the fact that $w_\loc (x_\locind)$ is a probability vector for all $\epsilon>0$ and all $x \in \mathbb{R}^d$.
\end{proof}

\noindent
Lemma \ref{lemma1} also establishes stability of the general Langevin sampler defined by (\ref{eq:propagator_general}) with localized $m_{\rm loc}(x;\epsilon)$ and $S(x;\epsilon) = 2\epsilon I$, i.e.,
\begin{equation} \label{eq:propagator_local}
X(n+1) = m_{\rm loc}(X(n);\epsilon) + \sqrt{2\epsilon}\,
\Xi(n),\qquad \Xi(n) \sim {\rm N}(0,I),
\end{equation}
for all step sizes $\epsilon>0$. Note that $X(n+1)$ is no longer in the convex hull of the data as the original unlocalized Schr\"odinger bridge (cf. \eqref{eq:update_ss}b), but instead is confined to $\mathcal{C}_M$ in expectation. The exact gradient structure of the conditional expectation value (\ref{eq:gradient structure}) does no longer hold for the localized $m_{\rm loc}(x;\epsilon)$. The next lemma shows that the localized sampler (\ref{eq:propagator_local}) remains geometrically ergodic.

\begin{lemma} \label{lemma2} 
Let us assume that the data generating density $\pi$ has compact support. Then the localized time-stepping method (\ref{eq:propagator_local}) possesses a unique invariant measure and is geometrically ergodic.
\end{lemma}

\begin{proof} Consider the Lyapunov function $V(x) = \|x\|^2$ and introduce the ball
\begin{equation}
\mathcal{B}_R = \{ x \in \mathbb{R}^d: \|x\| \le R\}
\end{equation}
of radius $R>0$. Since $m_{\rm loc}(x;\epsilon) \in \mathcal{C}_M$ and $\pi$ has compact support, one can find a radius $R>0$, which is independent of the training data $\mathcal{X}$, such that $\mathcal{C}_M \subset \mathcal{B}_R$ and
\begin{equation}
    \mathbb{E}[V(X(n+1))|X(n)] \le \lambda V(X(n))
\end{equation}
for all $X(n) \notin \mathcal{B}_R$ with $0\le \lambda < 1$. Furthermore, because of the additive Gaussian noise in (\ref{eq:propagator_local}), there is a constant $\delta>0$ such that
\begin{equation}
{\rm n}(x';m_{\rm loc}(x;\epsilon),2\epsilon I) \ge \delta 
\end{equation}
for all $x,x'\in \mathcal{B}_R$. Here ${\rm n}(x;m,C)$ denotes the Gaussian probability density function with mean $m$ and covariance matrix $C$. In other words, $\mathcal{B}_R$ is a small set in the sense of \cite{MeynTweedy}. Geometric ergodicity follows from Theorem 15.0.1 in \cite{MeynTweedy}. See also the self-contained presentation in \cite{MSH02}.
\end{proof}

\begin{rem}
We emphasize that, contrary to the unlocalized Schr\"odinger bridge sampler, the localized $m_{\rm loc}(x;\epsilon)$ is not restricted to the linear subspace of $\mathbb{R}^d$ spanned by the training data $\mathcal{X} \in \mathbb{R}^{d\times M}$ in case $M < d$. The localized sampler shares this desirable property  with the localized EnKF \citep{reich2015probabilistic,asch2016data,Evensenetal2022}.
\end{rem}

%
\subsection{Localized kernel-denoising} \label{sec:local KDE}
%
Localizing the KDE-based denoiser (\ref{eq:KDE drift}) follows along the same lines. We first note that a component-wise formulation of Tweedie's formula (\ref{eq:tweedie}) leads to
\begin{equation}
\partial_{x_\loc} \log \pi_\epsilon (x) = -\frac{1}{\epsilon} \left(x_\loc - \mathbb{E}[x_\loc'|x]\right),
\end{equation}
$\alpha = 1,\ldots,d$. Upon assuming the conditional independence relation
\begin{equation} \label{eq:cir}
\mathbb{E}[x_\loc'|x] = \mathbb{E}[x_\loc'|x_\locind]
\end{equation}
one finds that 
\begin{equation}
    \partial_{x_\loc} \log \pi_\epsilon (x)= \partial_{x_\loc} \log \pi_\epsilon (x_\locind).
\end{equation}
We approximate the restricted density $\pi_\epsilon (x_\locind)$ by the localized KDE estimator
\begin{equation}
\tilde \pi_\epsilon (x_\locind) \propto
\sum_{j=1}^M \exp \left(-\frac{1}{2\epsilon}\|x_\locind - x_\locind^{(j)}\|^2\right),
\end{equation}
which results in
\begin{equation}
\partial_{x_\loc} \log \pi_\epsilon (x) \approx -\frac{1}{\epsilon} \left( x_\loc - \sum_{j=1}^M
x_\loc^{(j)}\,{\tilde{w}}_\loc^{(j)}(x_\locind)\right)
\end{equation}
with localized weights
\begin{equation}
\tilde w_\loc^{(j)}(x_\locind) := \frac{\exp \left(-\frac{1}{2\epsilon}\|x_\locind - x_\locind^{(j)}\|^2\right)}{\sum_{j=1}^M \exp \left(-\frac{1}{2\epsilon}\|x_\locind - x_\locind^{(j)}\|^2\right)}
\end{equation}
for $j=1,\ldots,M$. We collect these weights in the vector $\tilde w_\loc(x;\epsilon)\in \mathbb{R}^M$.
One finally obtains the following localized KDE update step for $X_\loc(n+1/2)$:
\begin{equation} \label{eq:localized kernel denoising}
    X_\loc(n+1) = D_\loc(X(n+1/2);\epsilon) := 
    \mathcal{X}_\loc \,\tilde w_\loc(X_\locind(n+1/2)), \qquad \alpha = 1,\ldots,d,
\end{equation}
which can be employed whenever (\ref{eq:cir}) holds to sufficient accuracy.

%
\section{Localised Schr\"odinger bridge sampler for temporal stochastic processes}
\label{sec:temporal}
%

In this section, we consider temporal stochastic processes $Z(t_k) \in \mathbb{R}^s$ with $t_k = k \Delta t$ and $k=0,\ldots,K$, and assume that $M$ realizations 
\begin{equation}
    x^{(j)} = \{Z^{(j)}(t_k)\}_{k=0}^K \in \mathbb{R}^d,\qquad d = (K+1)s,
\end{equation}
$j = 1,\ldots,M$, of such a process have become available. We furthermore assume that the generating process is Markovian, i.e., $Z(t_{k+1})$ is conditionally independent of all $Z(t_{l})$ with $l<k$ and $l>k+1$. Such a setting provides a perfect application of the localization strategy proposed in Section \ref{sec:lsbs}. More specifically, we obtain  the following subsets for localization in terms of the entries $x_\loc$ of the augmented state vector $x \in \mathbb{R}^d$: $\Lambda(\loc) = \{1,\ldots,s\}$ for $\loc \in \{1,\ldots,s\}$ and
\begin{equation}
\Lambda(\loc) = \{sl+1,\ldots,s(l+2)\}
\end{equation}
for $\loc \in \{s(l+1)+1,\ldots,s(l+2)\}$ and $l=0,\ldots,K-2$.

As a numerical illustration we consider the bimodal stochastic differential equation (SDE)  
\begin{equation} \label{eq:bistable LD}
\frac{d}{dt} Z(t) = -Z(t)^3 + Z(t) + \sqrt{0.2}
\frac{d}{dt}B(t), \qquad Z(0) \sim {\rm N}(0,1),
\end{equation}
where $B(t)$ denotes standard Brownian motion. Here $s=1$ and we sample solutions in time-intervals of length $\Delta t = 5$ over $K=100$ intervals; hence the dimension of the augmented state vector $x \in \mathbb{R}^d$ becomes $d = 101$. The training data consists of $M = 1,000$ independent realizations, which were obtained with a small step-size EM algorithm applied to (\ref{eq:bistable LD}).

Results from the localized Schr\"odinger bridge split-step sampler \eqref{eq:update_ss_loc} with constant diffusion $S(x;\epsilon) = 2\epsilon I$ with $\epsilon = 0.0025$ and $N = 25,000$ generated samples can be found in Figures \ref{fig:proc1} and \ref{fig:proc2}, respectively.
The numerical results demonstrate that the localized Schr\"odinger bridge sampler can successfully generated samples for this rather high-dimensional ($d = 101$) and nonlinear problem given only $M=1,000$ training samples. 

We also implemented the localized KDE-based denoiser (\ref{eq:localized kernel denoising}). The results are virtually indistinguishable from the results obtained from the localized Schr\"odinger bridge sampler. 

\begin{figure}[!htp]
    \centering
    \includegraphics[scale=0.4]{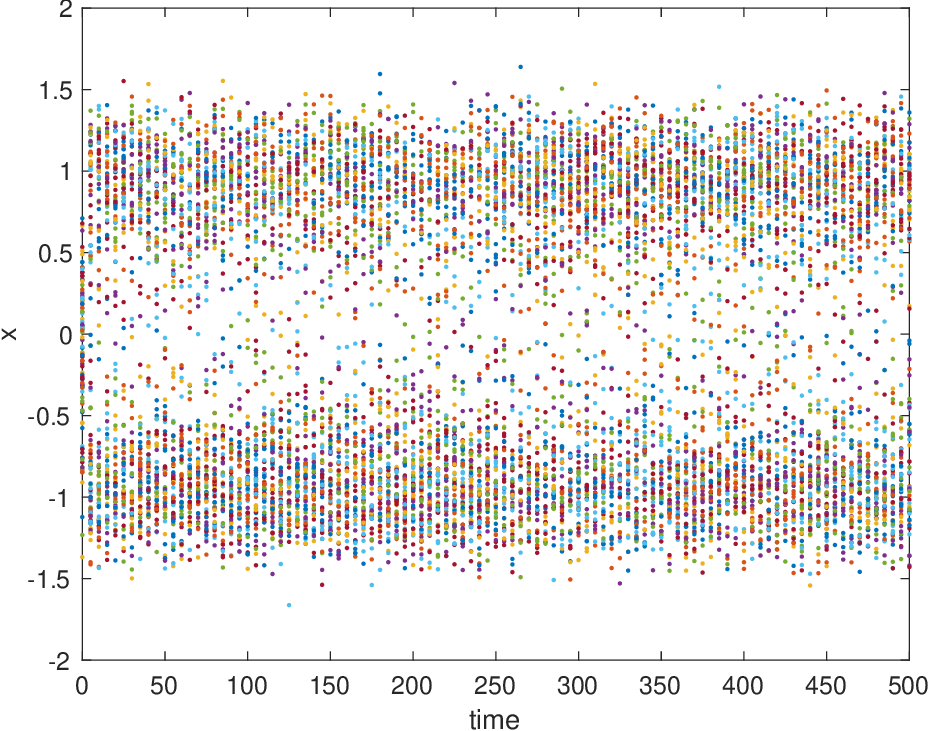}$\qquad$
    \includegraphics[scale=0.4]{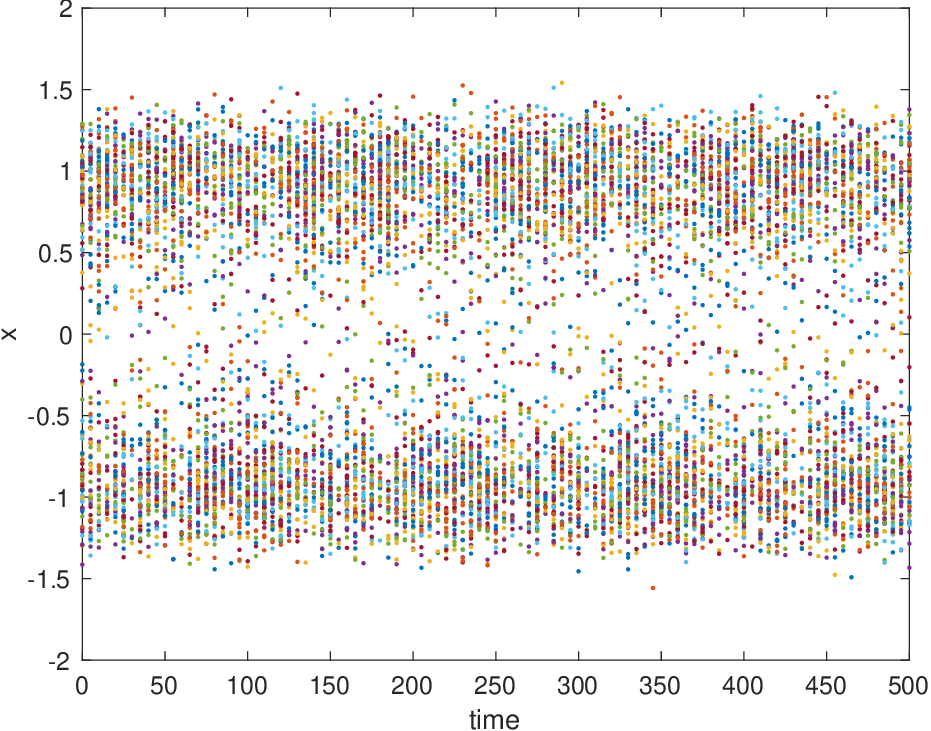}
    \caption{Generated trajectories for the bimodal SDE \eqref{eq:bistable LD} using the localized Schr\"odinger bridge split-step sampler with constant diffusion \eqref{eq:update_ss_loc}. Left panel: $100$ trajectories out of $M=1,000$ training samples; right panel: $100$ trajectories out of $N=25,000$ generated samples. The computed transition rates (relative number of sign changes along trajectories) agree well with $9\%$ for the training data and $11\%$ for the generated data.} \label{fig:proc1}
\end{figure}

\begin{figure}[!htp]
    \centering
    \includegraphics[scale=0.45]{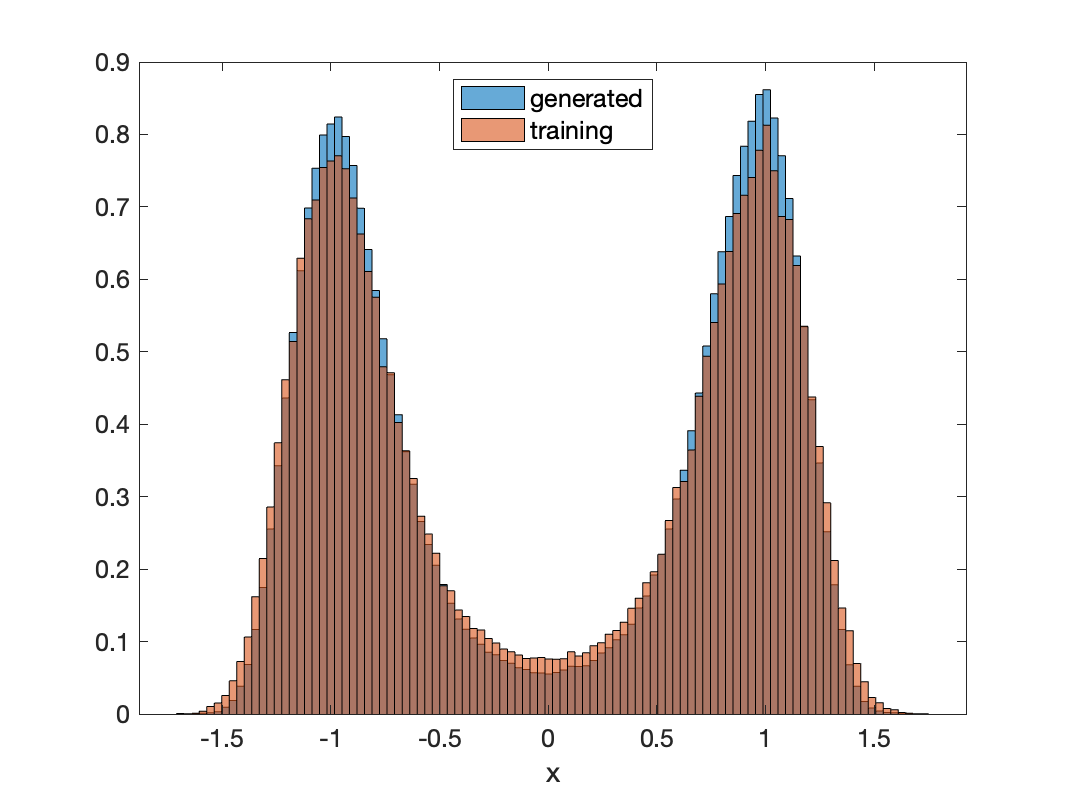}
    \caption{Normalized empirical histograms of training and generated data for the bimodal SDE \eqref{eq:bistable LD} using the localized Schr\"odinger bridge split-step sampler with constant diffusion \eqref{eq:update_ss_loc}. We show results over all $1,000$ training and $25,000$ generated data points. The invariant distribution of the bimodal SDE is well reproduced by the generated data; the dispersion of the generated data in each of its two modes being slightly smaller than the one from the training data, which has also been observed for the split-step scheme in Figure \ref{fig:multiGauss2}.} \label{fig:proc2}
\end{figure}

%
\section{Conditional localized Schr\"odinger bridge sampler}
\label{sec:cond}
%

As for the standard Schr\"odinger bridge sampler \cite{GLRY24} the localized sampler lends itself to conditional sampling. Consider samples $x^{(j)}=(z^{(j)},\psi^{(j)})$ for $j=1,\dots,M$. The localized Schr\"odinger bridge sampler described in Section~\ref{sec:lsbs} and Algorithm~\ref{alg:LSBS} allows us to learn the joint probability measure $\nu({\rm d}z,{\rm d}\psi)$. To draw samples from the conditional probability measure $\nu({\rm d}\psi|z)$ we may use the localized conditional probability vector $t_\loc(x_\locind)$ and the Sinkhorn weights $v_\loc$ obtained from the samples $x^{(j)}$ , i.e., executing lines 1-6 in Algorithm~\ref{alg:LSBS}. Conditional sampling is achieved by ensuring that at each sampling step the $z$-component of the generated samples $X(n)$ are set to the value $z^\ast$ on which we wish to condition. To achieve this we add the conditioning assignment, $Z(n) \gets z^\ast$, between lines 11 and 12 of Algorithm~\ref{alg:LSBS}. We demonstrate the performance of the conditional sampler for the multi-scale Lorenz-96 model in the next subsection.  

%
\subsection{Conditional sampling for a closure problem } 
%
We apply the conditional localized Schr\"odinger bridge sampler to the multi-scale Lorenz-96 model 
for $K$ slow variables $z_k$ which are each coupled to $J$ fast variables $y_{j,k}$ and evolve according to
\begin{subequations}\label{eq:L96}
\begin{align}
\frac{d}{dt}z_k &= -z_{k-1}(z_{k-2}-z_{k+1})-z_k+F-\frac{hc}{b}\sum\limits_{j=1}^{J}y_{j,k}, \\
\frac{d}{dt}y_{j,k} &= -cby_{j+1,k}(y_{j+2,k}-y_{j-1,k})-cy_{j,k}+\frac{hc}{b}z_k  
\end{align}
\end{subequations}
with periodic boundary conditions $z_{k+K}=z_k$, $y_{j,k+K}=y_{j,k}$ and $y_{j+J,k} = y_{j,k+1}$. This $d=K(J+1)$-dimensional model was introduced as a caricature for the mid-latitude atmospheric dynamics \cite{Lorenz96}. The degree of time-scale separation is controlled by the parameter $c$. The ratio of the amplitudes of the large-scale variables $z_k$ and the small-scale variables $y_{j,k}$ is controlled by the parameter $b$. The slow and fast dynamics are coupled with coupling strength $h$. The parameter $F$ denotes external forcing. As equation parameters we choose $K=12$ and $J=24$, i.e. $d=300$, and $F=20$, $c=b=10$ and $h=1$ as in \cite{Wilks05,ArnoldEtAl13,GottwaldReich21a}. These parameters lead to chaotic dynamics with a maximal Lyapunov exponent of $\lambda_{\rm{max}}\approx 18.29$ in which the fast variables experience temporal fluctuations which are $10$ times faster and $10$ times smaller than those of the slow variables. This corresponds to the regime of strong coupling in which the dynamics is driven by the fast sub-system \citep{HerreraEtAl10}.

In the climate science and other disciplines one is typically only interested in the slow large-scale dynamics. A direct simulation of the multi-scale system \eqref{eq:L96}, however, requires a small time step adapted to the fastest occurring time scale, making long term integration to resolve the slow dynamics computationally infeasible. Scientists hence aim to design a computationally tractable model for the slow variables only in which the effect of the fast dynamics is parameterized. This is the so called closure or subgrid-scale parameterization problem. In particular, we seek a model of the form
\begin{align}
\frac{d}{dt}z_k &=G_{k}(z) + \psi_k(z),
\label{eq:L96_closure}
\end{align}
for $z = (z_1,z_2, \ldots, z_K)$ with $G_{k}(z)= -z_{k-1}(z_{k-2}-z_{k+1})-z_{k}+F$. We assume that scientists have prior physics-based knowledge about the resolved vector field $G_{k}(z)$ but lack knowledge of the closure term $\psi_k(z)$ which parametrizes the effect of the fast unresolved dynamics. The closure term may be deterministic or stochastic, depending on the choice of equation parameters in \eqref{eq:L96}. We will employ the localized Schr\"odinger bridge sampler to generate samples of the closure term $\psi(z)$ conditioned on the current model state $z(t)$. The sampler will be trained on $M$ samples $x^{(j)}=(z^{(j)},\psi^{(j)})$, $j=1,\ldots , M$, which consists of a time series with $x^{(j)}=x(t_j)$ with $t_j = j \Delta t$ and $\Delta t = 5\times  {10^{-3}}$.

We obtain $M=40,000$ samples $z^{(j)} \in \mathbb{R}^{K}$ by integrating \eqref{eq:L96} using a fourth-order Runge--Kutta method with a fixed time step $\delta t = 5\times  {10^{-4}}$ and collecting the state in time intervals of $\Delta t = 10\,\delta t$. Samples of the closure term $\psi^{(j)} \in \mathbb{R}^K$ are then determined from the samples $z^{(j)}$ 
via
\begin{align}
\psi^{(j)} := \frac{z^{(j+1)} - z^{(j)}}{\Delta t} - G(z^{(j)}) , 
\end{align}
for $j=1,\cdots,M-1$. 
This defines $M-1$ samples $x^{(j)}=(z^{(j)},\psi^{(j)}) \in \mathbb{R}^{2K}$ for $j=1,\ldots,M-1$ to be used to train the Schr\"odinger bridge sampler. 

To numerically integrate the closure model (\ref{eq:L96_closure}) in terms of the state vector $z \in \mathbb{R}^K$ we employ an Euler discretization
\begin{align}
z(m+1) = z(m) + (G(z(m))  + \psi(m|z(m))) \, \Delta t
\label{eq:L96_closure_EM}
\end{align}
with a time step $\Delta t$. At each time step $m\ge0$ we generate a sample $\psi(m|z(m))$ conditioned on the current state $z(m)$. These samples should be uncorrelated to the samples drawn at the previous time step. This is achieved by running the localized Schr\"odinger bridge sampler conditioned on $z^\ast = z(m)$ at each time step $m$ for $n_c=100$ decorrelation steps (cf.~Algorithm~\ref{alg:LSBS}).

For the localized Schr\"odinger bridge sampler we employ a parameter of $\epsilon=0.1$ and consider a nearest neighbor localization with $\Lambda(\loc)=\{\loc-1,\loc,\loc+1,\loc, K+\loc-1,K+\loc,K+\loc+1\}$ with the obvious periodic extensions for $\loc=1$ and $\loc=d$. To account for the varying ranges of $z$ and $\psi$ when estimating the matrices \eqref{eq:tij_local} and \eqref{eq:tij_local2} for fixed parameter $\epsilon$, we replace the standard Euclidean product with a scaled one where we divide the inner product in the $z$-variables by $\sigma_z$ and the $\psi$-variables by $\sigma_\psi$, where $\sigma_\psi^2$ and $\sigma_x^2$ denote the climatic variances of the slow variables and the closure term, respectively, estimated from the samples $x^{(j)}$. 

Figures~\ref{fig:L96} and \ref{fig:L96_2} show a comparison of the outputs of the localized Schr\"odinger bridge sampler with data obtained from simulating the full multi-scale Lorenz-96 system \eqref{eq:L96}. We show results for the covariance of the slow variables $z$, obtained from the samples $z^{(j)}$ of the full multi-scale Lorenz-96 system \eqref{eq:L96}, and of the discretization of the closure scheme \eqref{eq:L96_closure_EM}. 
We show in Figure~\ref{fig:L96} a comparison of the empirical histograms of $z$ obtained by integrating the closure model \eqref{eq:L96_closure_EM} with the original samples $\{z^{(j)}\}_{j=1}^M$ which were obtained from a simulation of the full multi-scale Lorenz-96 system \eqref{eq:L96}. The non-stiff trained stochastic closure model \eqref{eq:L96_closure_EM} is able to reproduce the actual histogram well. We further show a scatter plot of the stochastic closure term $\psi$ obtained from the full Lorenz-96 system \eqref{eq:L96} and obtained from the localized conditional Schr\"odinger bridge. The closure term is well represented by the localized conditional Schr\"odinger bridge. In Figure~\ref{fig:L96_2} we show the entries of the rows of the empirical covariance matrix, centered about $k=6$ employing periodicity of the system. It is seen that our localized sampler reproduces the covariance structure of the full system very well. We further show that the temporal autocorrelation structure of the Lorenz-96 system is well reproduced by the localized conditional Schr\"odinger bridge. 

\begin{figure}[!htp]
    \centering
    \includegraphics[scale=0.35]{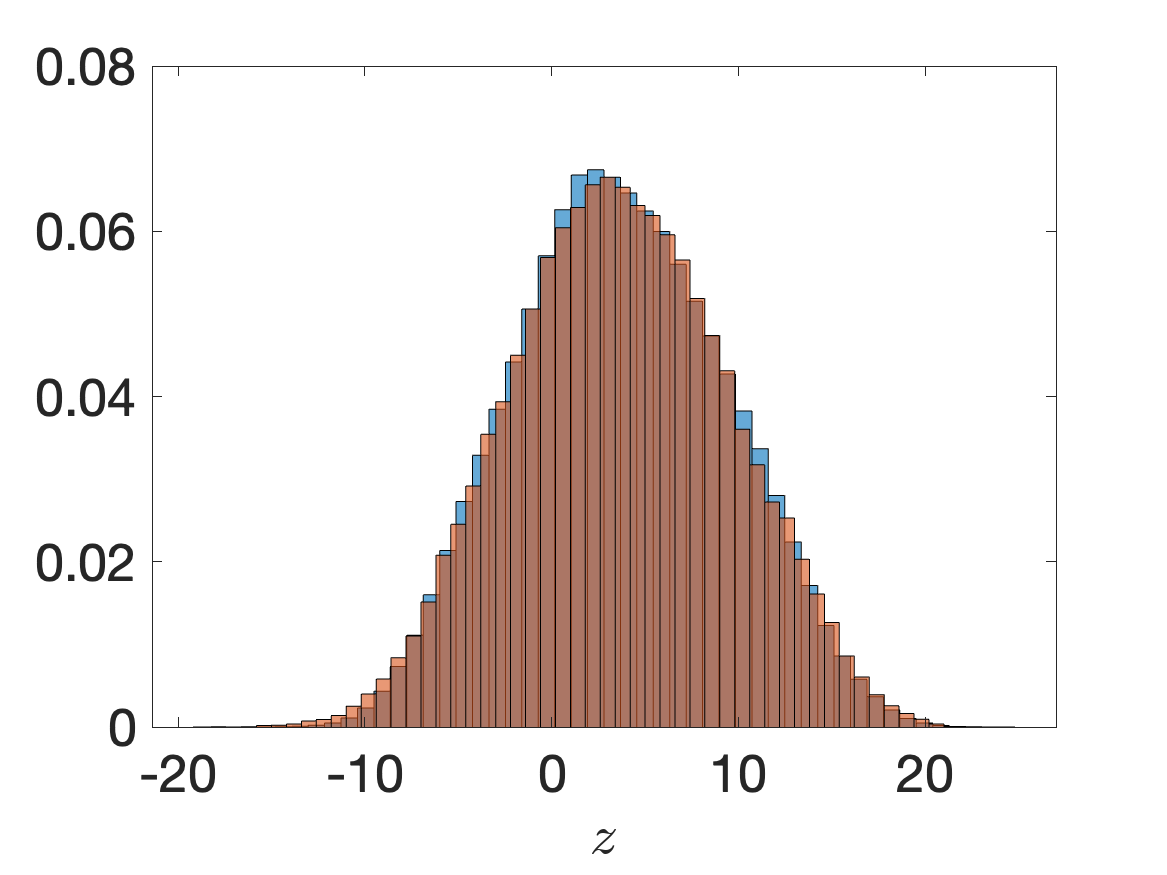}
    \includegraphics[scale=0.35]{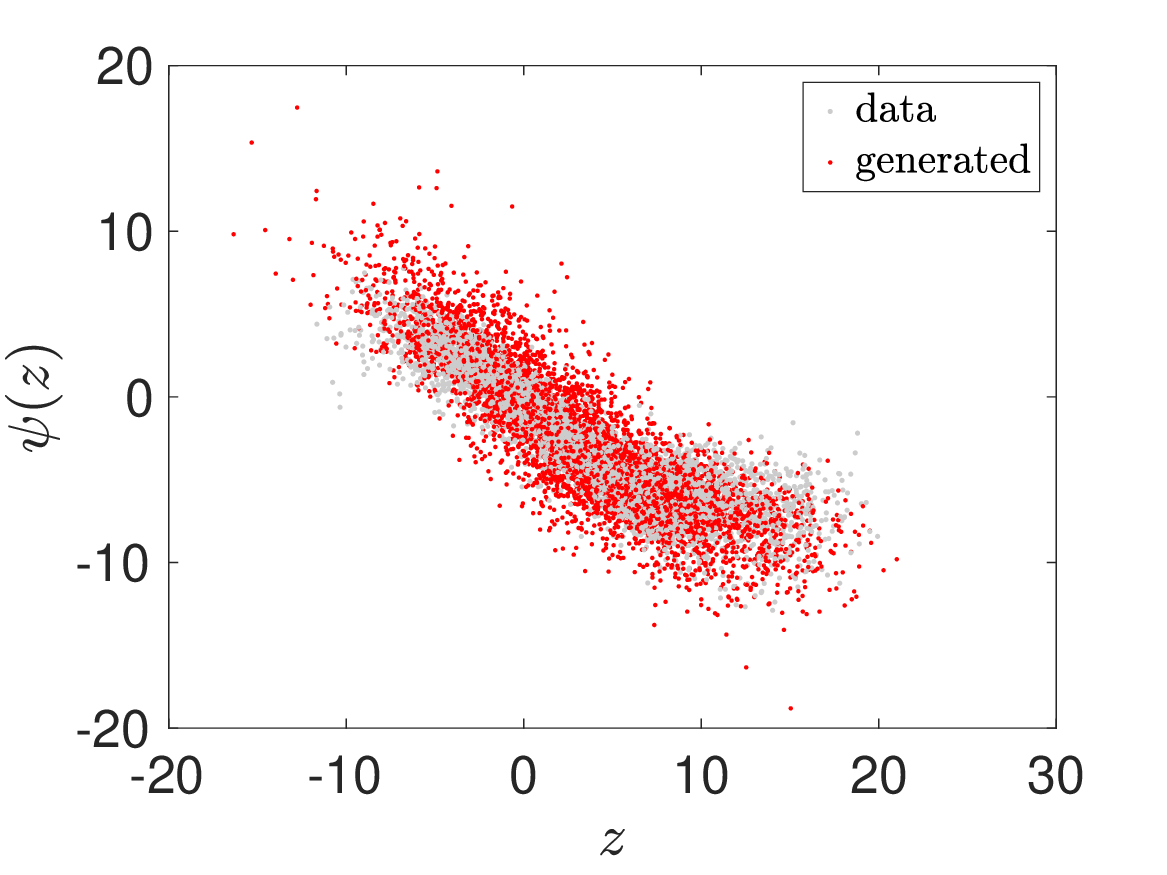}
    \caption{Comparison of the samples obtained from the localized Schr\"odinger bridge sampler and given samples drawn from the multi-scale Lorenz-96 system \eqref{eq:L96} using nearest neighbor localization with $\Lambda(\loc)=\{\loc-1,\loc,\loc+1\, K+\loc-1,K+\loc,K+\loc+1\}$ with the obvious periodic extensions for $\loc=1$ and $\loc=d$. We consider $40,000$ new and given samples. Left: Empirical histograms. Right: Scatter plot of the closure term $\psi$ as a function of $z$.}
    \label{fig:L96}
\end{figure}
\begin{figure}[!htp]
    \centering
    \includegraphics[scale=0.35]{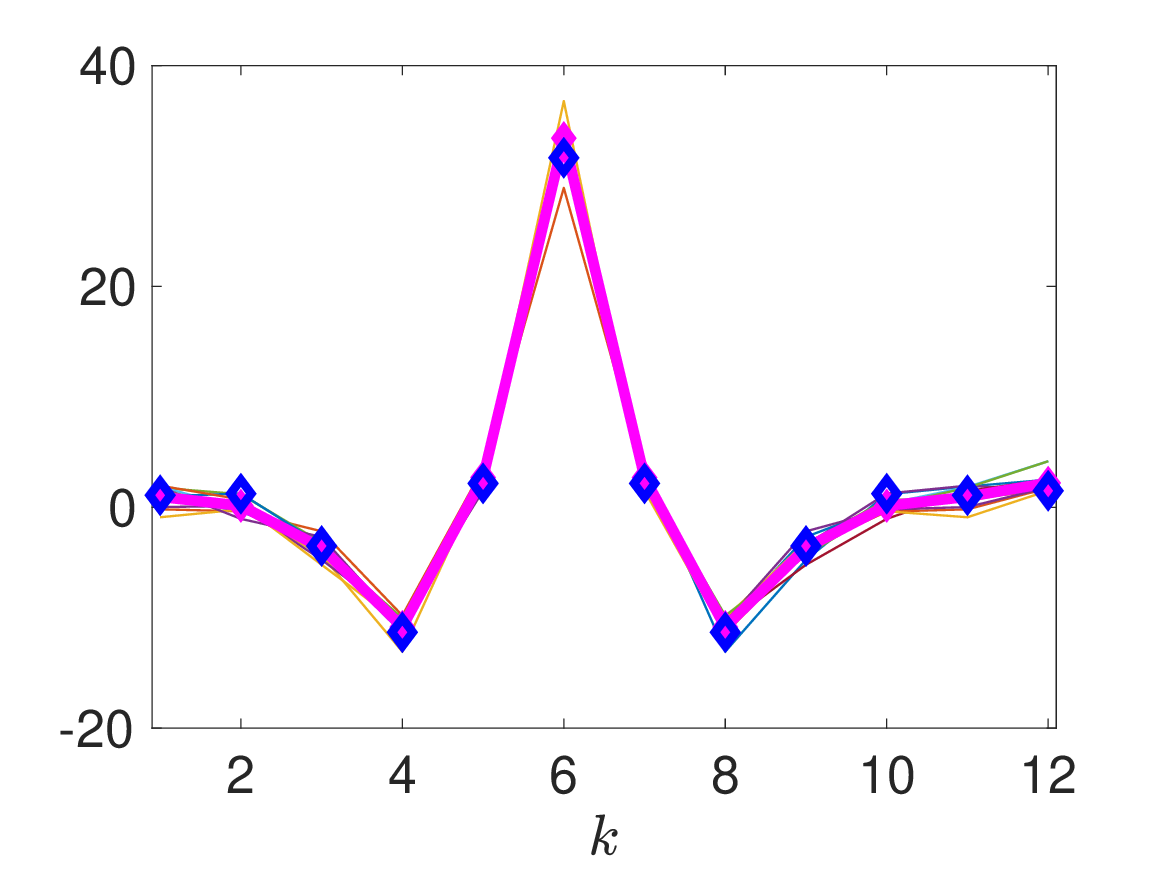}
    \includegraphics[scale=0.35]{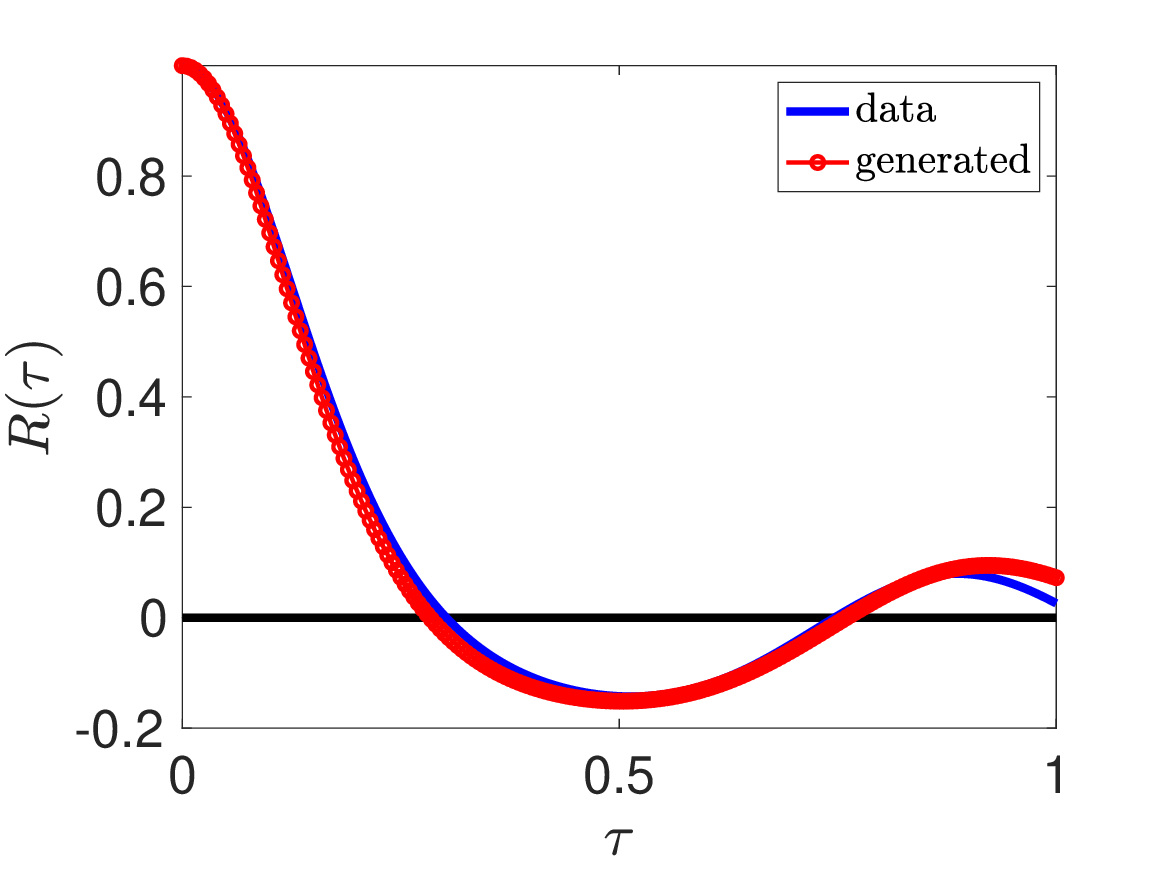}
    \caption{Comparison of the samples obtained from the localized Schr\"odinger bridge sampler and given samples drawn from the multi-scale Lorenz-96 system \eqref{eq:L96} using nearest neighbor localization with $\Lambda(\loc)=\{\loc-1,\loc,\loc+1\, K+\loc-1,K+\loc,K+\loc+1\}$ with the obvious periodic extensions for $\loc=1$ and $\loc=d$. We consider $40,000$ new and given samples. Left: Centered rows of the empirical covariance matrix for $z$. The magenta line denotes the mean over all rows. The blue markers denote the empirical covariance for the given samples. Right: Autocovariance function $R(\tau)$. }
    \label{fig:L96_2}
\end{figure}
%

%
\section{Conclusions} \label{sec:con}
%

The construction  of the previously proposed Schr\"odinger bridge sampler \cite{GLRY24} is fraught with an unfavorable dependency in the dimension $d$. The required number of samples scales for a desired accuracy exponentially on the underlying intrinsic dimensionality of the data \cite{WR20}. We have shown here that for data which satisfy conditional independence one can successfully employ localization to express the single Schr\"odinger bridge problem for  $d$-dimensional data to $d$ localized Schr\"odinger bridge problems of smaller size $d_\loc \ll d$. The localized Schr\"odinger bridge sampler can be used to generate samples from an unknown prior and readily lends itself to conditional sampling and Bayesian inference. 

We have numerically demonstrated the advantage of localization for several examples. We considered a Gaussian distribution for which the inverse covariance matrix has tri-diagonal structure, a bimodal SDE and a conditional sampling problem of determining a closure term in a nonlinear multi-scale system.

We have established theoretically that the proposed sampler is stable and geometric ergodic under relatively mild conditions. The stability of our sampler allows for applications to data drawn from a singular measure which arise when data are concentrated on a lower-dimensional manifold. This sets it apart from score-generative models which rely on Tweedie's formula and the differentiability of a regularized measure. 

We have established several connections with other sampling strategies. The Schr\"odinger bridge sampler was shown to be closely related to kernel-based denoising. The Schr\"odinger bridge sampler, however, has the advantage that it can employ a data-aware noising step, which was demonstrated to be advantageous in Section \ref{sec:Gauss}, and can be constructed using a variable bandwidth \cite{GLRY24}, which is desirable with training data that involve data-sparse regions in the state space. Further, while this work has focused on overdamped Langevin dynamics as a mean of sampling from a distribution, the methodology generalizes to more general formulations of score-generative and diffusion modeling \cite{diffusion1,diffusion2,diffusion3,diffusion4} and transformers \cite{Transformer,SABP22}. We have shown that the conditional mean of a Schr\"odinger bridge sampler is formally akin to self-attention in transformer architectures and that localization naturally leads to multi-head self attention. It will be interesting to further explore these connections. 

The framework of localized Schr\"odinger bridges lends itself to numerous applications. In particular, we mention here sequential data assimilation \cite{reich2015probabilistic,asch2016data,Evensenetal2022}, feedback particle filter and homotopy methods \citep{SR-meyn13,reich2011dynamical,PR21}, which are implemented utilizing Schr\"odinger bridges, and interacting particle sampling methods, which rely on grad-log density estimators such as (\ref{eq:loss}) \citep{MRO20}. Finally, the proposed localized conditional estimator $m_{\rm loc}(x)$ as well as its KDE-based variant (\ref{eq:localized kernel denoising}) could be of independent interest for MMSE denoising \cite{MD24}.

\smallskip
\smallskip
\smallskip

\paragraph{Acknowledgements.}
This work has been funded by Deutsche Forschungsgemeinschaft (DFG) - Project-ID 318763901 - SFB1294. GAG acknowledges funding from the Australian Research Council, grant DP220100931.


\bibliographystyle{abbrvnat}

%
\bibliography{references.bib}
%


\end{document}